\documentclass{article}
\usepackage[utf8]{inputenc}

\usepackage[nonatbib,final]{neurips_2022}

\usepackage{algorithm}

\usepackage{amsmath}
\usepackage{amssymb}
\usepackage{amsfonts}
\usepackage{amsthm}
\usepackage{todonotes}

\usepackage{mathtools}

\usepackage{enumitem}
\usepackage{caption}
\usepackage{subcaption}
\usepackage{stmaryrd}
\usepackage{epsfig}
\usepackage{graphicx}
\usepackage{dsfont}

\usepackage[T1]{fontenc}
\usepackage{adjustbox}
\usepackage{hyperref}
\usepackage[nameinlink, sort&compress, capitalize]{cleveref}
\usepackage{comment}
\usepackage{subcaption}

\newtheorem{theorem}{Theorem}
\newtheorem{property}{Property}
\newtheorem{lemma}{Lemma}
\newtheorem{definition}{Definition}
\newtheorem{assumption}{Assumption}
\newtheorem{remark}{Remark}
\newtheorem{proposition}{Proposition}

\DeclareMathOperator*{\argmin}{arg\min}
\DeclareMathOperator{\Id}{Id}

\newcommand{\eqdef}{\coloneqq}
\newcommand{\Lspace}{\ensuremath{L^2( \left[ 0,1 \right], V)}}
\renewcommand{\v}{\ensuremath{\Vert}}
\newcommand{\iset}{\ensuremath{1 \leq i \leq N}}
\renewcommand{\r}{\ensuremath{\mathbb{R}}}
\newcommand{\C}{\ensuremath{\mathbb{C}}}

\usepackage{xspace}
\newcommand{\modelname}{RKHS-NODE\xspace}
\newcommand{\modelnames}{\modelname{}s\xspace}

\newcommand{\tset}{\ensuremath{\left[ 0,1 \right]}}
\newcommand{\qint}{\ensuremath{q_{\text{\upshape int}}}}
\renewcommand{\qint}{\ensuremath{r}} % q_int was not very beautiful I think 
\newcommand{\h}{\ensuremath{\mathcal{H}}}

\newcommand{\myparagraph}[1]{\vspace{-1mm}\noindent\textbf{#1} }

\binoppenalty=\maxdimen
\relpenalty=\maxdimen

\begin{document}

\title{On global convergence of ResNets: From finite to infinite width using linear parameterization}

\author{
    Raphaël Barboni\\
    ENS - PSL Univ.\\
    \texttt{raphael.barboni@ens.fr}
    \And
    Gabriel Peyré\\
    CNRS and ENS - PSL Univ.\\
    \texttt{gabriel.peyre@ens.fr}
    \And
    Fran\c{c}ois-Xavier Vialard\\
    LIGM, Univ. Gustave Eiffel, CNRS\\
    \texttt{francois-xavier.vialard@u-pem.fr}
}

\maketitle

\begin{abstract}
Overparameterization is a key factor in the absence of convexity to explain global convergence of gradient descent (GD) for neural networks. Beside the well studied lazy regime, infinite width (mean field) analysis has been developed for shallow networks, using  convex optimization techniques. To bridge the gap between the lazy and mean field regimes, we study Residual Networks (ResNets) in which the residual block has linear parameterization while still being nonlinear. Such ResNets admit both infinite depth and width limits, encoding residual blocks in a Reproducing Kernel Hilbert Space (RKHS). In this limit, we prove a local Polyak-Lojasiewicz inequality. Thus, every critical point is a global minimizer and a local convergence result of GD holds, retrieving the lazy regime. In contrast with other mean-field studies, it applies to both parametric and non-parametric cases under an expressivity condition on the residuals. Our analysis leads to a practical and quantified recipe: starting from a universal RKHS, Random Fourier Features are applied to obtain a finite dimensional parameterization satisfying with high-probability our expressivity condition. %Experiments showcase the relevance of this convergence analysis and suggest that, despite being simpler, our model also enjoys generalization performances on par with those of classical ResNets.
\end{abstract}

%\smallskip
%\noindent

\section{Introduction}

State of the art supervised learning methods are based on deep neural networks, sometimes heavily overparameterized, which perfectly fit training data or even noisy data while exhibiting good generalization properties. Such a behaviour appears as a paradox and  questions the established theory of ``bias-variance trade-off''~\cite{belkin2019reconciling}. 
That an overparameterized model can fit data perfectly comes as no surprise but this capability does not explain the observed generalization properties. Towards a better understanding of it, one first needs to understand the optimization procedure in the parameter space that selects the interpolation map. This question is tightly linked with the parameterization of the space of maps that are explored and state of the art parameterizations have emerged in the past years. One key architecture that is ubiquitous in deep learning are skip connections, heavily used in \emph{Residual Neural Networks } (ResNets)~\cite{he2016deep} and it has led to state of the art results in supervised learning. ResNets actually allow one to consider a very large number of layers~\cite{zagoruyko2016wide}.

\myparagraph{Continuous models.}
Passing to the limit of infinite depth allows the connection with continuous models (Neural ODE) for which theoretical methods and new algorithms can be designed \cite{chen2018neural, vialard2020shooting}. Indeed, the similarities between ResNet architectures and discrete numerical schemes motivated the introduction of a continuous neural ODE
\begin{align} \label{continuous_resnet}
     \Dot{z}_t = v(W_t,z_t) \quad \forall t \in \tset,
\end{align}
where $W \in L^2(\tset, \r^m)$ is the parameter of the model and $v : \mathbb{R}^m \times \mathbb{R}^q \rightarrow \mathbb{R}^q$ is a \emph{residual transformation} whose output is the \emph{residual term}. These models correspond to limiting models of a discrete ResNet whose depth $D$ tends to infinity. Therefore, their study  brings a theoretical framework for understanding deep ResNet architectures, and more generally very deep NNs~\cite{e_mean-field_2019,e_barron_2021}. Moreover, their mathematical analysis is facilitated since it allows one to leverage a large body of works and tools from analysis and in particular the theory of optimal control~\cite{pontryagin1987mathematical}. Conversely, methods from numerical analysis can bring inspiration for designing new architectures and new optimization algorithms~\cite{lu2018beyond}.

\myparagraph{RKHS parameterization.}
Most often in the literature studying the training properties of ResNets, the considered residual transformations are \emph{Multi-Layer Perceptrons (MLP)}~\cite{du2019gradient,allen2019convergence,hardt_identity_2018}. Those consist in the composition of several trained linear layers alternatively composed with a non-linear activation function. A 2-layer MLP with width $\qint$ reads:
\begin{align} \label{MLP}
    v : \left( (W,U),z \right) \mapsto W \sigma (U z) ,
\end{align}
where $U \in \r^{\qint \times q}$ and $W \in \r^{q \times \qint}$ are the parameters for the ``hidden'' and the ``visible'' layer respectively and $\sigma : \r \rightarrow \r$ is a non-linear \emph{activation function} applied component-wise. Popular activation functions are for example the ReLU or the Swish function. Provided with these activations, MLPs enjoy a nice universal approximation property as shown in the seminal work of Barron~\cite{barron_universal_1993}.

In contrast, we consider here a setting where the residual term is linear w.r.t. the parameters while still being nonlinear w.r.t the inputs. Given a feature map $\varphi : \r^q \rightarrow \r^{\qint}$, we consider as space of residuals the space:
\begin{align} \label{residual_set}
    V \eqdef \lbrace v : z \mapsto W \varphi(z) | W \in \r^{q \times \qint} \rbrace ,
\end{align}
where the matrices $W \in \r^{q \times \qint}$ are the trained parameters. Compared to~\cref{MLP}, this can be seen as an MLP where the hidden layer is fixed by introducing the feature map $\varphi : z \mapsto \sigma (U z)$ for some feature matrix $U$. As is standard, the gradient of some loss $L$ w.r.t. $W$ is computed in the sense of the Frobenius metric on the set of matrices:
\begin{align} \label{frobenius_metric}
    \forall\, W, W' \in \r^{q \times \qint}, \ \langle W, W' \rangle = \text{Tr}(W^\top W') . 
\end{align}
Such an $L^2$ penalization induces a metric structure on $V$ through the identification $v \leftrightarrow W$ in~\cref{residual_set}:
\begin{align} \label{residual_metric}
    \forall\, v,v' \in V, \ \langle v, v' \rangle_V \eqdef \langle W, W' \rangle .
\end{align}
As a finite dimensional space of continuous maps, $V$ has the structure of \emph{Reproducing Kernel Hilbert Space} (RKHS). Moreover, as pointed out in~\cite{bach2017breaking}, the space $V$ has a natural infinite width limit or mean field limit which is an infinite dimensional RKHS.

In this paper, we are interested in understanding the convergence properties of Gradient Descent (GD) on a ResNet model for which the residual layers are encoded in a -- possibly infinite-dimensional -- vector-valued RKHS $V$. For $V$ as in~\cref{residual_set}, we stress out that, as the metric on $V$ is induced by the one on $\r^{q \times \qint}$, GD on $V$ for this metric is strictly equivalent to GD on $\r^{q \times \qint}$ with the Frobenius metric. Our model is defined as follows:

\begin{definition}[RKHS Neural ODE (RKHS-NODE)] \label{def:rkhs_resnet}
Let $V$  be a RKHS of vector-fields over $\r^q$ and $A \in \r^{q \times d}$, $B \in \r^{d' \times q}$. Then for $v \in \Lspace$ and a data input $x \in \mathbb{R}^d$, the \modelname's output is  $F(v, x) \eqdef  B z_1$,
where $z$ is the solution to the \emph{forward problem}
\begin{align} \label{rkhs_forward}
%    \left\{
 %       \begin{array}{rcl}
            \Dot{z}_t  = & v_t(z_t) 
            \quad\text{and}\quad
            z_0  = A x .
 %       \end{array}
 %   \right.
\end{align}
The variable $v$ will thereafter be called \emph{control parameter}.
\end{definition}

\begin{remark}
Note that the matrices $A$ and $B$ are fixed and only the control parameter $v$ is trained. However, we argue that our approach can be simply adapted to the case where $B$ is trained, following for example the proof of~\cite{nguyen_proof_2021}. Training $A$ seems more challenging as the model is highly non-linear w.r.t. this parameter.
\end{remark}

%Let us discuss the shortcomings and potential benefits of this simplified model.
%

\myparagraph{Relevance of the RKHS model.}
%Considering such a simplified model comes with shortcomings as well as potential benefits.
%
The main difference between the model of~\cref{def:rkhs_resnet} and standard ResNets is linearity in the parameters of the residual blocks. As a comparison, a 2-layer MLP is nonlinear w.r.t. the parameters of the hidden layers. 
%While it is admittedly a simplified setting,
However, this linearity assumption does not impact the expressivity of the model, but only its training dynamic. 
\textbf{(i)} Indeed, considering $V$ to be a Random Features approximation (c.f.~\cref{RFF}) of some universal RKHS, the residual blocks are as expressive as a 2-layer MLP since both are dense in the space of continuous functions.
\textbf{(ii)} Up to the cost of adding a supplementary variable, the dynamical system parameterized by a 2-layer MLP can be expressed as a model which is linear w.r.t. its parameters~\cite[Section 3.2]{vialard2020shooting}. Only the training dynamic between these two architectures differs.
Also, this assumption of linearity in the parameters also prevents the use of normalization layers. In this direction, \cite{zhang2018fixup} has shown that ResNets without normalization but proper initialization of the weights can lead to robust training and similar performance on the train set than standard ResNets.
%\todoraph{j'ai changé la phrase suivante}
%Finally, \cref{def:rkhs_resnet} encompasses a whole variety of model: for $\varphi = \Id$ one recovers linear ResNets, whereas for a non-linear feature map $\varphi$, due to composition of the residuals, the model is highly non-linear w.r.t. both input and parameters.
Finally, the model of~\cref{def:rkhs_resnet} still retains the effect of depth and the nonlinearity w.r.t. the input. Due to composition of these residual blocks the model's output is still highly non-linear w.r.t. parameters.
For these reasons, we consider this model as an important step towards the study of the general case.

\begin{comment}

\myparagraph{Relevance of the RKHS model.}
Considering such a model comes with shortcomings as well as potential benefits.
%
The main assumption that differs from standard ResNets is linearity in the parameters of the residual blocks. As a comparison, a 2-layer MLP is nonlinear w.r.t. the parameters of the hidden layer. While it is admittedly a simplified setting, the model of~\cref{def:rkhs_resnet} still retains the effect of depth and the nonlinearity w.r.t. the input. Indeed, considering $V$ to be a Random Features approximation (c.f.~\cref{RFF}) of some universal RKHS, the residual blocks are as expressive as a 2-layer MLP as both are dense in the space of continuous functions. Moreover, due to composition of these residual blocks the model's output is still highly non-linear w.r.t. parameters.
This assumption of linearity in the parameters also prevents the use of normalization layers. In this direction, \cite{zhang2018fixup} has shown that ResNets without normalization but proper initialization of the weights can lead to robust training and similar performance on the train set than standard ResNets.
Therefore, we consider this model as a first step of study towards the general case.

\end{comment}

%We assess the practical applicability of our model for supervised learning tasks by performing experiments on real datasets in~\cref{sec:expe} but mostly consider its study as a first step towards a more general case.

In turn, this linearity in parameters naturally leads to an RKHS parameterization which has two important benefits on the theoretical side:
%
%\todo{Gab: ok -- mais du coup il manque une phrase de transition vers les 2 points ci dessous}
%
\textbf{(i)} Flows of vector-fields as implemented by our model in~\cref{rkhs_forward} have already been studied theoretically and for applications in image registration problems~\cite{younes_shapes_2010, beg2005computing,Niethammer2011}. Under some regularity assumptions on the considered RKHS $V$, one can show that the model's output corresponds to the invertible action of a diffeomorphism by composition on the input~\cite{trouve1998diffeomorphisms}. This property was already used in~\cite{salman2018deep} to implement models of \emph{Normalizing Flows}~\cite{kobyzev_normalizing_2020} with applications in generative modeling.
 %   \item 
\textbf{(ii)} There is an important literature in Machine Learning about Kernel methods~\cite{scholkopf2002learning}. In practice, various sub-sampling methods exist in order to approximate infinite-dimensional RKHSs with finite-dimensional spaces generated by \textit{Random Fourier Features }(RFF)~\cite{rahimi2007random,rahimi2008weighted}. Thereby, leveraging results on the approximation bound for RFF~\cite{sutherland2015error,sriperumbudur2015optimal}, we show that the expressiveness properties of universal kernels, such as the Gaussian kernel, can be efficiently recovered using residuals of the form~\cref{residual_set} with a finite number of neurons. 

To further support the practical applicability and the relevance of our model in comparison with standard architectures, we report in the supplementary material (\cref{sec:expe}) numerical experiments on MNIST and CIFAR10 datasets. They show that -- as predicted by our theory -- our model can be trained in these cases to almost zero loss. But more importantly, they show that our model is able to generalize well on the test dataset with performances that are similar to those of classical ResNets.

\myparagraph{Supervised learning.}
We consider a map $F$ from $\mathcal{H} \times \r^d$ to $\r^{d'}$ for some Hilbert space of parameter $\mathcal{H}$ (e.g. the model of~\cref{def:rkhs_resnet} with $\mathcal{H} = \Lspace$) and a training dataset consisting on a family of inputs $(x^i)_{\iset} \in (\r^d)^N$ and target outputs $(y^i)_{\iset} \in (\r^{d'})^N$. 
Then for every parameter $v \in \mathcal{H}$, we define the associated \emph{Empirical Risk} as:
\begin{align} \label{empirical_risk}
    L(v) \eqdef \frac{1}{2N} \sum_{\iset} \v F(v,x^i) - y^i \v^2 .
\end{align}

\begin{remark} \label{PL_losses}
For simplicity we consider here the Euclidean square distance as a loss on the output space $\r^{d'}$, but our results generalize to any smooth loss satisfying a Polyak-Lojasiewicz inequality (c.f.\cite{bolte_characterizations_2009}), e.g. any smooth strongly convex loss.
\end{remark}

Training the model $F$ then amounts to finding a parameter $v^* \ \in \argmin_{v \in \mathcal{H}} L(v)$. In order to perform such an \emph{empirical risk minimization (ERM)} task we consider GD on $v$. For a small step size $\eta$, for some initialization $v^0 \in \mathcal{H}$ and for every discrete time step $k \in \mathbb{N}$, the training dynamic reads:
\begin{align*}
%    \left\{
%        \begin{array}{rcl}
            v^{k+1}   =  v^k - \eta \nabla L(v^k) .
%            \quad\text{and}\quad
%            v^0  \in \mathcal{H}.
%        \end{array}
%    \right.
\end{align*}
Note that we do not consider any additional regularizing term on the loss. In classical supervised learning one would seek for a minimizer of the ``regularized'' loss $L(v) +  \lambda \mathcal{R}(v)$, with $\lambda > 0$ a constant and $\mathcal{R}$ a coercive regularization function. We are here interested in the non regularized setting, i.e. $\lambda = 0$, often used in practice.
In this case, the generalization property of the computed map is argued to potentially come from the optimization method that shall select an adequate minimizer of the loss. This implicit regularization depends on the choice of the optimization method~\cite{neyshabur2017implicit}.
%Analysis of its convergence then only relies on the properties of $V$ such as expressivity and regularity.

\section{Related works and contributions} \label{sec:contributions}

Recently, several works have addressed the problem of proving convergence of (stochastic) GD in the training of NNs.
If the convergence properties of GD are well understood for NNs that are linear w.r.t. input~\cite{hardt_identity_2018,bartlett2018gradient,zou2019global}, it is not the case for non-linear NNs.
In~\cite{li2017convergence,li2018learning,du2018gradient}, the authors focus on the training of ``shallow'' two layers fully connected NNs and establish convergence of GD in an overparameterized setting where width of the intermediary layer scales polynomially with the size $N$ of the dataset. More recently, with the same
%problem
setup, \cite{zhou2021local} showed that the neurons of a teacher network are recovered by a student network optimized with GD as long as the width of the student network is higher than the teacher's one. Formally, their analysis is  similar to ours as the result holds if the loss at initialization is already sufficiently low and the proof relies on Polyak-Lojasiewicz inequalities verified by the loss landscape. 

\myparagraph{Infinite depth.} 
The works of~\cite{du2019gradient,allen2019convergence,zou2019global,lee2019wide,zou_gradient_2020,liu2020linearity,chen2020much,nguyen_proof_2021} extend those results to arbitrary deep NN in the overparameterized setting. Specifically, the results in \cite{du2019gradient,allen2019convergence,liu2020linearity} apply to deep ResNets. The best result seems to be achieved in~\cite{nguyen_proof_2021}, with convergence as soon as the last layer has a width $m = \Omega(N^3)$ and at best with linear width. A common feature for those works is to rely on the fact that, for a sufficiently high number of parameters, the model can be well approximated by a linear model corresponding to its first order expansion around the initialization. In~\cite{chizat2019lazy} this phenomenon, called ``lazy regime'', is attributed to an inappropriate scaling of the parameters. On the other hand, \cite{liu_loss_2021,liu2020linearity} refer to this phenomenon as ``linear'' or ``kernel regime'' and relate it the constancy of the \emph{Neural Tangent Kernel (NTK)} introduced in~\cite{jacot_neural_2021}. However, in all those works the width of intermediary layers has to depend on the depth $D$ of the network. Therefore, these results do not apply to the training of the model in~\cref{continuous_resnet}, corresponding to the limit $D \rightarrow +\infty$.

\myparagraph{Infinite width.} 
The other direction of over-parameterization, analyzed in several works~\cite{mei2018mean,chizat2018global,mei2019mean,javanmard_analysis_2020,lu2020mean,fang2021modeling,pham2020global} is to consider the limit of infinitely wide layers. In such a ``mean-field'' setting, the model is parameterized by the distribution of the parameters at each layer. In~\cite{chizat2018global,mei2018mean,mei2019mean,javanmard_analysis_2020} the training dynamic is analyzed as a gradient flow in the Wasserstein space~\cite{ambrosio2008gradient}, showing that the only stationary distributions are global minimizers of the empirical risk. In~\cite{fang2021modeling} a similar result is showed for deep NN with an arbitrary number of infinitely wide layers. In~\cite{chizat2021sparse,akiyama2021learnability}, local linear convergence towards the global optimum is shown for two layers NNs in a teacher-student setup with regularized loss. Finally, \cite{lu2020mean} analyzes the convergence of continuous ResNets with infinitely wide residual layers and shows that every critical point is a global minimizer of the empirical risk. We stress out that these
%convergence
results only apply to infinitely wide NNs. It is not clear if this mean-field limit extends to the parametric setting of MLPs with the Euclidean metric on their parameters. In contrast, a RKHS structure naturally arises when considering a linear parameterization of the residuals. \Cref{ass:admissibility} and~\cref{ass:universality} can be satisfied both in a parametric setting with a finite number of features and in a mean-field setting limit where the residuals are generated by a universal kernel.

\myparagraph{Contributions.} We show convergence results for GD in the training of \modelnames (see~\cref{def:rkhs_resnet}). These correspond to infinitely deep continuous ResNets with linear parameterization of the residuals.
% \begin{itemize} \item 
Our first main contribution, in~\cref{sec:main}, shows that under some regularity and expressivity assumptions on the residuals, the associated empirical risk satisfies a (local) Polyak-Lojasiewicz~\cref{prop:PL_rkhs}. A consequence is~\cref{thm:main}, which states global convergence of GD towards a global optimum  (zero training loss) under the condition that the loss at initialization is already sufficiently low. In the limit where the loss at initialization is arbitrarily small, we recover a linear regime as described in~\cite{liu_loss_2021,liu2020linearity}.
%    \item 
Our second contribution, in~\cref{sec:convergence}, shows how this condition for global convergence can be enforced using suitably chosen first and last linear layers. Thereafter, we show how the assumptions of~\cref{thm:main}, can be satisfied for RKHSs generated by a finite number of Random Features, with high probability over the choice of these features.
%
%\todo{Etre + precis (ref thm + donner qq details du statement)}
%
For any dataset $(x^i,y^i)_{\iset} \in (\r^d \times \r^{d'})^N$, we conclude in~\cref{thm:global_convergence} to convergence of GD towards a global minimum of~\cref{empirical_risk} with high probability when the width of the layers scales polynomially w.r.t. the size of the dataset $N$ and the inverse input data separation $\delta^{-1}$.

%\todoraph{ci-dessous paragraphe sur le lien avec les modèles linéaires}
Finally, we point out that some of our results can be seen as a generalization of existing results concerning convergence of GD for the training of linear NNs~\cite{hardt_identity_2018,bartlett2018gradient,zou2019global}. We explain in~\cref{sec:linear} how, following the line of our analysis, one can for example recover~\cite[Theorem 3.1.]{zou2019global}. However, if~\cref{def:rkhs_resnet} encompasses linear ResNets as a special case, we stress that~\cref{thm:main} applies to a way larger class of models.
%In particular, it is an important feature of our work to provide an analysis structure that can be easily adapted to different models or data-structure.

% \end{itemize}
%\todofx{Raph: j'ai rajouté ce paragraphe pour les contrib}
%Finally, we assess the practical applicability of those results on real datasets in~\cref{sec:expe}. While trained on MNIST we observe that in addition to reaching low training loss, our model generalizes well on a test dataset. Moreover, the number of parameters being fixed, the performances of our model are of the same order as when hidden layers are trained.

\myparagraph{Notations.} In what follows $\v . \v$ denotes the Euclidean $\ell^2$ norm for vectors and the Frobenius norm for matrices. For matrices the spectral norm is denoted $\v . \v_2$, the smallest (resp. greatest) singular value is denoted $\sigma_{\min}$ (resp. $\sigma_{\max}$) and for symmetric matrices the smallest (resp. greatest) eigenvalue is denoted $\lambda_{\min}$ (resp. $\lambda_{\max}$). Given some Hilbert space $\mathcal{H}$, the functional Hilbert space $L^2(\tset, \h)$ is denoted $L^2(\h)$ or $L^2$ when there is no ambiguity. The notation $\mathcal{O}$ (resp. $\Omega$) means asymptotically inferior (resp. superior) up to multiplicative constant.

\section{Analysis of convergence for overparameterized models} \label{sec:overparameterized}

In this section, we review methods for analyzing the convergence of overparameterized machine learning models based on~\cite{liu_loss_2021,liu2020linearity}. We refer to~\cref{sec:proof_overparameterized} for detailed proofs of the statements. 
%A parametric machine learning model is described as the action of a parameter $v$ in a Hilbert space $\h$ on an input object $x$, belonging to an input space $\mathcal{N}$. The resulting object $y$, belongs to the output space $\mathcal{M}$ of dimension $d'$. This action is represented by:
%\begin{align*}
%    F : 
%    \left\{
%        \begin{array}{ccc}
%            \h \times \mathcal{N} & \rightarrow & \mathcal{M}  \\
%            (v,x) & \mapsto & F(v,x) =: y ,
%        \end{array}
%    \right. 
%\end{align*}
%where we assume $F$ to be differentiable w.r.t. $v$.

%Provided a differentiable loss function $\ell : \mathcal{M} \times \mathcal{M} \rightarrow \mathbb{R}$, we are interested in the task of minimizing the associated empirical risk: 
%\begin{align} \label{pb:general_emp}
%    \text{Find} \ v^* \in \argmin_{v \in \mathbb{R}^m} L(v) \eqdef \frac{1}{N }\sum_{i=1}^N \ell (F(v,x^i) , y^i)
%\end{align}
%for a family of labeled input data $(x^i)_{1 \leq i \leq N} \in \mathcal{N}^N$ and objective data $(y^i)_{1 \leq i \leq N} \in \mathcal{M}^N$. Such an optimization task is performed by considering the GD dynamic :
%\begin{align} \label{gradient_descent}
%    v^{k+1} = v^k - \eta \nabla L(v^k), \ \forall k \in \mathbb{N},
%\end{align}
%with an initialization $v^0 \in \h$ and a step-size $\eta > 0$.
As presented above, we consider an optimization over the variable $v$ in some Hilbert space $\mathcal{H}$, with fixed input and output data, say $v \mapsto F(v) \eqdef [F(v,x_i)]_{i=1,\ldots,N}$. Therefore, the empirical risk is a function of the parameters $v \in \mathcal{H}$. We say that the model is \emph{overparameterized} whenever the dimension $\text{dim}(\mathcal{H})$ of the parameter space  is much larger than the dimension of the output space of $F(v)$, here $d'N$. The \modelname model defined $F$ in~\cref{def:rkhs_resnet} falls into this category as $\mathcal{H}$ is the infinite dimensional functional space $L^2(\tset, V)$.

\subsection{A (local) Polyak-Lojasiewicz property}

When dealing with overparameterized models, one cannot expect the loss to be convex but one expects the model to perfectly fit the data, that is to reach the global minimum value of $0$. In fact, for a sufficient number of parameters, the loss landscape typically possesses a continuum of infinitely many global minima and is non-convex in any neighbourhood of a global minima~\cite{liu_loss_2021}. One thus rather needs to rely on a set of functional inequalities allowing to control the decrease rate of the loss along GD~\cite{lojasiewicz1982trajectoires,bolte_characterizations_2009}.

\begin{definition}[(local) Polyak-Lojasiewicz property] \label{def:PL_gen}
Let $L: \mathcal{H} \rightarrow \r_+$ be a differentiable function. We say that $L$ satisfies a (local) Polyak-Lojasiewicz (PL) property if there exist positive continuous functions $m,M : \r_+ \rightarrow \r_+^*$ s.t. for every $v \in \mathcal{H}$
\begin{align} \label{PL_ineq_gen}
%    \begin{array}{ccc}
%        \v \nabla L(v) \v^2 & \geq & 2 M(\v v \v) L(v) ,  \\
        2 m(\v v \v) L(v) \leq \v \nabla L(v) \v^2 \leq  2 M(\v v \v) L(v)  .
%    \end{array}
\end{align}
\end{definition}

Such functional inequalities have already shown to be relevant for proving convergence guarantees in the training of NNs~\cite{frei2021proxy}. A first consequence for a loss $L$ which satisfies the (local) PL property of~\cref{def:PL_gen} is that it does not admit any spurious local minima but only global minima. Also, if the training dynamic is bounded, then $m$ and $M$ are uniformly lower- and upper-bounded along the dynamic, implying that $L$ decreases at a linear rate. In most cases,  $m$ and $M$ are degenerate when $\v v \v \rightarrow +\infty$. When the dynamic is not bounded, $L$ can thus decrease to $0$ slower than at a linear rate or even converge towards a strictly positive limit.

\subsection{Local convergence result}

Because of the degeneracy of $m$ and $M$, it is in general not possible to conclude an unconditional convergence of GD towards a global minimizer of the empirical risk.
However, PL inequalities are sufficient to prove convergence when the problem is not too hard to solve, that is when the loss at initialization is not too high.
Moreover, when using gradient descent stepping, one needs to make a supplementary smoothness assumption on the empirical risk $L$. This ensures that the loss decreases at each gradient step for a sufficiently small step size.

\begin{definition}[Smoothness, Definition 2 of~\cite{liu_loss_2021}] \label{def:smoothness}
Let $\beta \geq 0$ be a constant. We say that the function $L : \mathcal{H} \rightarrow \r$ is $\beta$-smooth if for every $v,v' \in \mathcal{H}$:
$
    \left| L(v') - L(v) - \langle \nabla L(v), v'-v \rangle \right| \leq \frac{\beta}{2} \v v'-v \v^2 .
$
\end{definition}

The local PL property combined with this smoothness assumption then gives a local convergence result for the convergence of GD towards a global minimizer of the empirical risk.

\begin{theorem}[Theorem 6 of~\cite{liu_loss_2021}] \label{thm:convergence_gen}
Let $L : \mathcal{H} \rightarrow \r_+$ be a loss function satisfying a local PL property with local constants $m$ and $M$. Let $v^0 \in \mathcal{H}$ and $R \geq 0$ be such that
\begin{align} \label{init_cond_gen}
    2 \sqrt{2}\frac{\sqrt{M(\v v^0 \v + R)}}{m(\v v^0 \v + R)}\sqrt{L(v^0)} \leq R .
\end{align}

%\todo{modified the statement for readability + space}
Furthermore, assume that $L$ is $\beta$-smooth within the ball $B(v^0,R)$. Then for a step size $\eta \leq \beta^{-1}$, GD with initialization $v^0$ and step size $\eta$ converges towards a global minimizer of $L$ with a linear convergence rate and inside a ball of radius $R$. More precisely, for every $k \geq 0$:
\begin{align} \label{convergence_rate_gen}
    L(v^k) \leq (1 - m(\v v^0 \v + R) \eta )^k L(v^0) \quad \text{and} \quad \v v^k - v^0 \v \leq R, \ \forall k \geq 0 .
\end{align}
\end{theorem}

\section{Properties of \modelname} \label{sec:main}

In this section we analyze the convergence of GD in the training of the infinitely deep ResNet model of~\cref{def:rkhs_resnet}. Note that such a model is overparameterized in depth as the parameter space is the infinite dimensional space $\Lspace$ and overparameterization can also come from width when the RKHS is high (or even infinite) dimensional.  Therefore, our proof of convergence %leverages the tools presented in the previous section and in particular 
heavily relies on a PL property verified by the empirical risk.

Recall that we consider the training of deep ResNets with a linear parameterization of the residuals. The set of residuals is as in~\cref{residual_set} with the metric of~\cref{residual_metric} induced by the Frobenius metric (\cref{frobenius_metric}). This provides $V$ with a RKHS structure~\cite{aronszajn1950theory}, whose associated kernel is given for any $z,z' \in \r^q$ by
$
    K(z,z') \eqdef \langle \varphi(z), \varphi(z') \rangle \Id_q, 
$
and whose associated feature map is given by $\varphi$.

\begin{remark}
The definition of $\langle ., . \rangle_V$ in~\cref{residual_metric} requires $\text{\upshape Span}( \varphi(\r^q) ) = \r^{\qint}$ to associate each $v \in V$ to a unique $W \in \r^{q \times \qint}$. This is satisfied by all the feature maps $\varphi$ we consider in the following.
\end{remark}

Given a training dataset composed of input data points $(x^i)_{\iset} \in (\r^d)^N$ and of target data points $(y^i)_{\iset} \in (\r^{d'})^N$ we are interested in the task of minimizing the empirical risk of~\cref{empirical_risk} by GD over $v$. Analogously to back-propagation in discrete NNs architectures, the gradient of $L$ can be expressed thanks to a backward equation derived by adjoint sensitivity analysis~\cite{pontryagin1987mathematical}.

\begin{property} \label{prop:grad}
Let $L$ be the empirical risk in~\cref{empirical_risk} associated with the \modelname model with a quadratic loss. Let $K$ be the kernel function associated with the RKHS $V$. Then $L$ is differentiable on $\Lspace$, with for every $v \in \Lspace$, 
$
    \nabla L(v) = \sum_{i=1}^N K(.,z^i) p^i,
$
where for each index $i \in \llbracket 1, N \rrbracket$, $z^i$ is the solution of~\cref{rkhs_forward} with initial condition $A x^i$ and the \emph{adjoint variable} $p^i$ is the solution to the \emph{backward problem}:
\begin{align} \label{rkhs_backward}
        \Dot{p}^i_t = - Dv_t(z^i_t)^\top p^i_t
        \quad \text{and} \quad
        p^i_1 = - \frac{1}{N} B^\top ( B z^i_1 - y^i ) .
\end{align}
\end{property}

%\todo{removed a line}
%This explicit formulation of the gradient is directly used to prove the PL property.

\subsection{PL property of \modelname}

Following the line of proof sketched in~\cref{sec:overparameterized}, we show how to derive PL inequalities of the form~\cref{PL_ineq_gen} for the empirical loss associated with the \modelname model. For that purpose we make a few assumptions about the RKHS $V$. The first one concerns its regularity and allows us to control the solutions of~\cref{rkhs_forward,rkhs_backward}.

\begin{assumption}[(strong) Admissibility] \label{ass:admissibility}
We say that the RKHS $V$ is \emph{(strongly) admissible} if it is continuously embedded in $W^{2,\infty}(\r^q, \r^q)$. More precisely, there exists a constant $\kappa > 0$ s.t.
\begin{align} \label{embedding}
    \forall v \in V, \quad
    \v v \v_\infty + \v Dv \v_{2,\infty} + \v D^2 v \v_{2,\infty} \leq \kappa \v v \v_V .
\end{align}
%We note this embedding $V\hookrightarrow W^{2,\infty}(\r^q,\r^q)$.
\end{assumption}

Assuming $V$ is embedded in $W^{1,\infty}(\r^q, \r^q)$ is natural to ensure the regularity of the flow generated by the control parameter~\cite{trouve1998diffeomorphisms, younes_shapes_2010} and suffices to prove convergence of a continuous gradient flow on the parameter $v$. \Cref{ass:admissibility} is a bit stronger because a supplementary smoothness result on the loss landscape is necessary to prove convergence of discrete GD (c.f.~\cref{def:smoothness}). In practice, 
%
%\todo{standard kernels encountered in ML are admissible, and their associated constant}
%
$\kappa$ can be computed for smooth kernels thanks to~\cref{prop:kappa} in~\cref{sec:convergence_proof}. For example, the RKHS associated with the Gaussian kernel $k : r \mapsto e^{-r^2/2}$ is (strongly) admissible with $\kappa = 2+\sqrt{3}$.

The second assumption is related to the expressiveness of $V$ and is a weaker form of the classical universality property of RKHSs.

\begin{assumption}[$N$-universality] \label{ass:universality}
Let $K$ be the kernel function associated with the RKHS $V$. For a family of points $(z^i)_{\iset} \in (\r^q)^N$, we define the associated kernel matrix as the block matrix $\mathbb{K}((z^i)_i) \eqdef (K(z^i, z^j))_{1 \leq i,j \leq N}$.
 \\
More precisely we assume for every $\delta > 0$:
\begin{align}\label{Lambda}
    \Lambda \eqdef \!\!\!  \underset{(z^i) \in (\r^q)^N}{\sup} \lambda_{\max} (\mathbb{K}((z^i)_i) ) < +\infty
    \quad\text{and}\quad
    \lambda(\delta^{-1}) \eqdef 
    \!\!\!\!\!\!   
    \underset{\substack{ (z^i) \in (\r^q)^N \\ \min_{ i \neq j }  \v z^i - z^j \v \geq \delta } }{\inf}
    \!\!\!  
    \lambda_{\min}(\mathbb{K}((z^i)_i) )  > 0 \,.
\end{align}
\end{assumption}

\Cref{ass:universality} is required in order to ensure the expressivity of our model, quantified by the conditioning of the kernel matrix $\mathbb{K}$ and by $\Lambda$ and $\lambda$.  The choice of the RKHS $V$ may thus have a significant impact on training.
In particular, satisfying~\cref{ass:universality} requires having $V$ of dimension $m \geq N$, but it can be satisfied for finite dimensional RKHSs of dimension $m \leq N^q$, for example by considering a polynomial kernel, or by RKHSs of dimension $m \geq poly(N,q)$ with high probability on the sampling of random features, as shown in~\cref{sec:convergence}. On the other hand, even though the existence of $\lambda$ follows from compactness arguments, it seems to be hardly analytically tractable even for classical kernels such as the Gaussian kernel.
Therefore, if, in theory, prior knowledge of the data distribution might allow to optimize the choice of kernel, we expect the selection of an optimal kernel to be an intractable problem in practice. Instead, cross-validation techniques can be used to select a suitable kernel.

\begin{remark}
For a RKHS $V$ as in~\cref{residual_set}, the properties of $V$ only depend on $\varphi$. An interesting example is when $\varphi : z \mapsto \sigma( U z)$ with $\sigma$ an activation function applied component-wise and $U$ a fixed feature matrix. In~\cref{sec:convergence} we show that, when considering the complex activation $\sigma : t \mapsto e^{- \imath t}$, both assumptions can be satisfied with high probability. On the other hand, \cref{ass:admissibility} is not satisfied when considering $\sigma = \text{ReLU}$ due to its non-smoothness at $0$.
\end{remark}

\begin{remark}
Note that $\Lambda$ could also be allowed to depend on some parameters, such as $\max \v z^i \v$. However, as it is a more critical aspect of our analysis, we prefer to highlight the dependency of $\lambda$ w.r.t. $\min_{i \neq j} \v z^i - z^j \v$. For all the RKHSs studied here we always have $\Lambda \leq N$.
%\todo{modified end of remark}
\end{remark}

The following PL property is satisfied by the risk~$L$. \Cref{prop:PL_rkhs} is proven in~\cref{subsec:proof_PL_rkhs}.

\begin{property}[\modelname satisfy PL] \label{prop:PL_rkhs}
Assume $V$ satisfies~\cref{ass:admissibility} with $\kappa$ and~\cref{ass:universality} with $\lambda$ and $\Lambda$. Let $L$ be the empirical risk in~\cref{empirical_risk} associated with the \modelname model of~\cref{def:rkhs_resnet}. Then $L$ satisfies the PL inequalities of~\cref{def:PL_gen} with $m$ and $M$ given by:
\begin{align} \label{PL_ineq_rkhs}
    M(R) = \frac{1}{N} \sigma_{\max}(B^\top)^2 \Lambda e^{2 \kappa R} \,, \quad
    m(R) = \frac{1}{N} \sigma_{\min}(B^\top)^2 \lambda \left( \sigma_{\min}(A)^{-1} \delta^{-1} e^{\kappa R} \right) e^{-2 \kappa R} ,
\end{align}
where $\delta \eqdef \min_{i \neq j} \v x^i - x^j \v$ is the data separation. 
\end{property}

\begin{proof}[Sketch of proof]
\Cref{ass:admissibility} can be used to have estimates on the solutions $z^i$ of the forward problem~\cref{rkhs_forward} and on the solutions $p^i$ of backward problem~\cref{rkhs_backward}. This gives for every indices $i,j \in \llbracket 1, N \rrbracket$ and every $t \in \tset$:
\begin{align*}
    \v z^i_t - z^j_t \v \geq \sigma_{\min}(A) \v x^i - x^j \v e^{- \kappa \v v \v_{L^2}} ,
\end{align*}
where $z^i$ solves~\cref{rkhs_forward} with initial condition $A x^i$, and:
\begin{align*}
    e^{-2 \kappa \v v \v_{L^2}} \v p^i_1 \v^2 \leq \v p^i_t \v^2 \leq e^{2 \kappa \v v \v_{L^2}} \v p^i_1 \v^2. 
\end{align*}
Moreover using the initial condition $p^i_1 = -\frac{1}{N} B^\top (B z^i_1 - y^i)$ we have:
\begin{align*}
    \frac{2 \sigma_{\min}(B^\top)^2}{N} L(v) \leq \sum_{i=1}^N \v p^i_1 \v^2 \leq \frac{2 \sigma_{\max}(B^\top)^2}{N} L(v) .
\end{align*}

Then denoting $\Tilde{p}_t$ the vector of stacked $p^i_t$ and using properties of RKHSs, we have for $t \in \tset$:
\begin{align*}
    \v \nabla L(v)_t \v^2 = \sum_{1 \leq i,j \leq N} (p^i_t)^\top K(z^i_t, z^j_t) p^j_t 
     = \langle \Tilde{p}_t, \mathbb{K} ((z^i_t)_i)) \Tilde{p}_t \rangle , 
\end{align*}
where $\mathbb{K}$ is the kernel matrix associated with the points $(z^i_t)_i$. This last equality gives the result using~\cref{ass:universality} and the previous estimates on $p^i$.
\end{proof}

Note that the degeneracy of the bounding functions $M,m$ as $R \rightarrow +\infty$ readily appears in~\cref{PL_ineq_rkhs}. Thus one should not expect these bounds to imply global convergence of GD without making any further assumption. Indeed, cases where GD fails to converge towards a global optimizer of the loss are observed in~\cite{bartlett2018gradient}, Section~6, with a setup corresponding to the model of~\cref{def:rkhs_resnet} with $V$ as in~\cref{residual_set} and $\varphi = \Id_{\r^q}$. 
Also, note that the data separation $\delta$ plays an important role in~\cref{prop:PL_rkhs} as it intervenes in the conditioning of the kernel matrix. In what follows, we always assume the data points to have a data separation lower-bounded by $\delta > 0$.

\subsection{Convergence of \modelname}

Thanks to the convergence analysis for overparameterized models detailed in~\cref{sec:overparameterized}, our main result follows as a consequence of the previous property. \cref{thm:main} is proven in~\cref{subsec:proof_main}.

\begin{theorem} \label{thm:main}
Let $V$ satisfy~\cref{ass:admissibility} with constant $\kappa$ and~\cref{ass:universality} with $\lambda, \Lambda$. Let $v^0$ be some initialization of the control parameter with $\v v^0 \v_{L^2} = R_0$ and assume there exists a positive radius $R \geq 0$ s.t.:
\begin{align} \label{rkhs_init_cond}
    \frac{\sqrt{8} \sigma_{\max}(B^\top) \sqrt{N \Lambda L(v^0)}  e^{3 \kappa (R + R_0)}}{\sigma_{\min}(B^\top)^2 \lambda (\sigma_{\min}(A)^{-1} \delta^{-1} e^{\kappa (R + R_0)}) } \leq R \,.
\end{align}

%\todo{modified the statement for readibility}
Then, for a sufficiently small step-size $\eta > 0$, GD with step-size $\eta$ converges towards a minimizer of the training loss at a linear rate and inside a ball of radius $R$. More precisely, for every $k \geq 0$:
\begin{align} \label{rkhs_convergence}
    L(v^k) \leq (1 - \eta \mu)^{k} L(v^0) , 
    \quad \text{and} \quad \v v^k - v^0 \v_{L^2} \leq R ,
\end{align}
where $\mu \!\eqdef\! \frac{1}{N} \sigma_{\min}^2(B^\top)  \lambda \left( \sigma_{\min}^{-1}(A) \delta^{-1} e^{\kappa (R + R_0)} \right) e^{-2 \kappa (R+R_0)}$.
\end{theorem}

\begin{comment}

\begin{proof}[Sketch of the proof]
Following~\cref{thm:main}, it only remains to show that the \modelname model is locally smooth.
%
Consider two control parameters $v, \Bar{v} \in \Lspace$ and associated solutions $z^i,\Bar{z}^i$ and $p^i, \Bar{p}^i$ of~\cref{rkhs_forward} and ~\cref{rkhs_backward}.  Using~\cref{ass:admissibility} one can derive estimates on $z^i - \Bar{z}^i$ and $p^i - \Bar{p}^i$. The result follows by
%controlling the quantity $\v \nabla L(v) - \nabla L(\Bar{v}) \v_{L^2}$ and by
establishing a bound of the form: $\v \nabla L(v) - \nabla L(\Bar{v}) \v_{L^2} \leq C \v v-\Bar{v} \v_{L^2}$.
\end{proof}

\end{comment}

As~\cref{thm:convergence_gen}, \cref{thm:main} is a local convergence result in which the condition in~\cref{rkhs_init_cond} expresses a threshold between two kinds of behaviours:
\textbf{(i)} if $L(v^0)$ is sufficiently small, the training dynamic converges towards a global minimizer. The limiting behaviour is when the l.h.s. of \cref{rkhs_init_cond} tends to $0$. Because of a regularizing effect of GD (i.e. that $\v v^k - v^0 \v_{L^2} \leq R$), the parameter stays in a ball of arbitrary small radius $R$ all along the training dynamic. In this limit, we recover a ``linear'' or ``kernel'' regime where the model is well approximated by its linearization at $v^0$~\cite{chizat2018global,liu2020linearity,jacot_neural_2021}. 
\textbf{(ii)} If $L(v^0)$ is too large, the result says nothing about the convergence of the GD. However, it is still observed in practice that the training dynamic often converges towards a global minimizer of the loss~\cite{zhang2021understanding}. Explaining this phenomenon in a general setting remains a challenging open question.

\section{Enforcing convergence with high dimensional embedding and finite width} \label{sec:convergence}

As~\cref{thm:main} is a local convergence result, it does not allow to conclude a general convergence behaviour of GD in the training of \modelname. In the following, we show how one can enforce the hypothesis of~\cref{thm:main} to be verified and prove two global convergence results. The first one relies on suitably choosing matrices $A$ and $B$ in order to satisfy~\cref{rkhs_init_cond} and applies in the case of infinite width, i.e. with residual layers in a universal RKHS. The second result recovers global convergence in a finite width regime, relying on a high number $\qint$ of Random Fourier Features.
%and a high embedding dimension $q$ to satisfy~\cref{rkhs_init_cond}.

For the sake of readability we only consider here the case where $V$ belongs to a restricted class of RKHSs and refer to~\cref{sec:convergence_proof} for more general results and complete proofs. For some positive parameter $\nu > 0$ we consider the Matérn kernel $k$ defined in~\cite{williams2006gaussian}:
\begin{align} \label{matern_kernel}
    \forall r \in \r_+, \ k(r) = \frac{2^{1 - \nu}}{\Gamma(\nu)} \left(\frac{\sqrt{2\nu}}{2 \pi} r \right)^{\nu} \mathcal{K}_{\nu} \Big( \frac{\sqrt{2\nu}}{2 \pi} r \Big) ,
\end{align}
where $\Gamma$ is the Gamma function and $\mathcal{K}_{\nu}$ is the modified Bessel function of the second kind. Equivalently, $k$ can be defined by its frequency distribution over $\r^q$ as:
\begin{align} \label{matern_frequency}
    \forall x \in \r^q, \ k(\v x \v) = \int_{\r^q} e^{\imath \langle x, \omega \rangle } \mu_q(\omega) \text{d}\omega \quad
    \text{with} \quad \mu_q(\omega) = C_{q,\nu} (1 + \frac{\v \omega \v^2}{2\nu} )^{-(\frac{q}{2}+\nu)}
\end{align}
and $C_{q,\nu}$ a normalizing constant. For every $q \geq 1$, such a function is known to define a structure of vector-valued RKHS $V_q$ over $\r^q$ corresponding to the Sobolev space $H^{\nu + q/2}(\r^q,\r^q)$~\cite{scholkopf2002learning,williams2006gaussian}.
%
%\todo{Gab: je me disais de peut etre ajouter une phrase ``Note that it important that this RKHS depends on the ambient dimension $q$, in particular that the regularity is larger than $q/2$ to ensure that it stays a valid reproducing space.''}
%
The associated kernel is given for every $z,z'\in \r^q$ by:
$
    K_q(z,z') = k(\v z-z' \v) \Id_q .
$
Note that it is important for this RKHS to depend on the ambient dimension $q$. In particular the Sobolev space $H^s(\r^q,\r^q)$ is a RKHS if and only if it has regularity $s > q/2$.
Assuming $\nu > 2$, $\mu_q$ further admits up to $4$ finite order moment implying that $k$ is four times differentiable at $0$~\cite{review_multivariate}. Then $V_q$ satisfies~\cref{ass:admissibility} with some constant $\kappa$ depending only on $\nu$ and given by~\cref{prop:kappa}:
\begin{align} \label{kappa}
    \kappa & = \sqrt{k(0)} + \sqrt{-k''(0)} + \sqrt{k^{(4)}(0)}  = 1 + \sqrt{\frac{\nu}{\nu-1}} + \sqrt{\frac{3\nu^2}{(\nu-1)(\nu-2)}}.
\end{align}
Also, $V_q$ satisfies~\cref{ass:universality} with $\Lambda \leq N$ and $\lambda$ depending a priori on $\nu$, $q$ and $N$.

Note that with this choice of scaling for $k$ and $\mu_q$, one recovers the Gaussian kernel $k : r \mapsto e^{-r^2/2}$ in the limit $\nu \rightarrow +\infty$~\cite{williams2006gaussian}. Thereafter we consider, $\nu \in (2, +\infty]$, the case $\nu = +\infty$ referring to the Gaussian kernel. We also assume that the data distribution is compactly supported.
%arguing that the proofs can easily be adapted to milder assumptions. 
In particular there exists some $r_0 \geq 0$ so that every input data $x$ verifies $\v x \v \leq r_0$.

\subsection{Global convergence with high-dimensional lifting}

We first show how~\cref{rkhs_init_cond} can be satisfied by considering appropriate embedding matrices $A$ and $B$. Doing so, the square distance between the data points, i.e. the model's loss, is preserved whereas the conditioning of the kernel matrix can be controlled.

%\todo{clarifier / dire avec des mots}
\begin{proposition} \label{prop:embedding} Let $\nu \in (2, +\infty]$, let $(x_i, y_i)_{1 \leq i \leq N} \in (\r^d \times \r^{d'})^N$  be a dataset with data separation $\delta > 0$ and let $R > 0$. There exist $q \geq 1$ and matrices $A \in \r^{q \times d}$, $B \in \r^{d' \times q}$ s.t. GD initialized at $v^0=0$ converges towards a zero-training-loss optimum in the training of \modelname. \\
In particular, \cref{rkhs_init_cond} holds with radius $R$ and $\kappa, \lambda, \Lambda$ associated with the RKHS~$V_q$.
\end{proposition}

As shown in the proof in~\cref{subsec:proof_q}, \cref{prop:embedding} still holds for small but non-zero initialization. We present here two ways of obtaining matrices $A$ and $B$ satisfying~\cref{rkhs_init_cond}:

\myparagraph{Scaling}
Consider
$A = \alpha (\Id_d, 0)^\top \in \r^{(d+d') \times d}$ and $B = \alpha (0, \Id_{d'}) \in \r^{d' \times (d+d')}$, for $\alpha > 0$. \\
We show in~\cref{subsubsec:scaling} that, in this setting, the l.h.s. of~\cref{rkhs_init_cond} scales as $\mathcal{O}(1/\alpha)$ and thus~\cref{thm:main} holds for large enough $\alpha$.
Moreover, observe that $q = d+d'$ is independent of $N$ and $\delta$ and such a regime can easily be implemented in practice. However, it has been shown that, although interpolation of the training data can be achieved as a consequence of a suitable rescaling of the parameters, this ``lazy regime'' can also lead to bad generalization properties~\cite{chizat2019lazy}.

\myparagraph{Lifting}
Consider for $q \geq 1$ the matrices:
$A_q \eqdef q^{-1/4}(\Id_d, ..., \Id_d, 0 )^\top \in \r^{q \times d}$ and $B_q \eqdef q^{1/4} (\Id_{d'}, 0 ... 0 ) \in \r^{d' \times q}$,
with $\lfloor q/d \rfloor$ copies of $\Id_d$ in $A_q$. This choice is motivated by the intention for these matrices to produce a high-dimensional lifting, which has been shown to improve on the expressivity of ResNets~\cite{dupont2019augmented}. We then show in~\cref{subsubsec:lifting} that~\cref{rkhs_init_cond} can be satisfied for $q = \Omega(N^4 + \delta^{-4}\log(N)^4)$. We do not expect our condition on $q$ to be optimal as we observe in experiments (see~\cref{sec:expe}) that a regime of linear convergence can be obtained for $q \ll N^4 + \delta^{-4} \log(N)^4$. However, we observe that increasing $q$ does improve on the convergence and generalization properties of our model (\cref{fig:8channels}).

\subsection{Global convergence with finite width}

In the preceding we showed that, in the case of an RKHS defined by a Matérn kernel, convergence of GD can be ensured for well-chosen matrices $A$ and $B$. However, for practical implementations, the form of the residual in~\cref{residual_set} forces us to consider RKHSs defined by feature maps. A way to overcome this difficulty and to benefit from the properties of a wide range of kernel functions is to consider an approximation by \emph{Random Fourier Features (RFF)}~\cite{rahimi2007random,rahimi2008weighted}.
More precisely, given $q \geq 1$, recall the definition of the Matérn kernel $k$ in terms of its frequency distribution $\mu_q$ over $\r^q$ in~\cref{matern_frequency} and for any sampling $\omega^1, ..., \omega^{\qint} \overset{iid}{\sim} \hspace{-0.4em} \mu_q$ of size $\qint$, consider the feature map:
\begin{align} \label{RFF}
    \varphi : z \in \r^q \mapsto \frac{1}{\sqrt{\qint}} (e^{\imath \langle z, \omega^j \rangle})_{1 \leq j \leq \qint} \in \C^{\qint} .
\end{align}

%\begin{remark}[Sampling] \label{rk:sampling}
%Note that $\mu_q$ identifies as the density of a $q$-variate $t$-distribution with shape parameter $2\nu$~\cite{review_multivariate}. Sampling over $\mu_q$ can be achieved using that for $Y \sim \mathcal{N}(0,\Id_q)$ and for $u$ distributed according to $\chi^2_{2 \nu}$, the chi-squared distribution with $2 \nu$ degrees of freedom, $Y/\sqrt{u/2\nu}$ is distributed according to $\mu_q$.  
%\end{remark}

In other words, considering the complex activation $\sigma : t \mapsto e^{\imath t}$ applied component-wise and $U \eqdef (\omega^1 | \ldots | \omega^{\qint}) \in \r^{q \times \qint}$ the feature matrix, we have
$\varphi(z) = \qint^{-1/2} \sigma ( U^\top z)$. Recall that such a feature map defines a structure of RKHS on
$
    \Hat{V}_q \eqdef \lbrace W \varphi(.) \ | \ W \in \r^{q \times \qint} \rbrace
$.
Such a $\Hat{V}_q$ can be viewed as a finite-dimensional approximation of the universal RKHS $V_q$ as it is associated with the kernel function $\Hat{K}_q (z,z') \eqdef \Hat{k}(z,z') \Id_q$, with:
\begin{align*}
    \Hat{k}(z,z') \eqdef \langle \varphi(z), \varphi(z') \rangle
    = \frac{1}{\qint} \sum_{j=1}^{\qint} e^{\imath \langle z-z', \omega^j \rangle}
    \xrightarrow[]{\qint \rightarrow + \infty} k( \v z-z' \v) \ \text{a.s.}
\end{align*}
Given any $q \geq 1$, we show that, with high probability over the choice of features, $\Hat{V}_q$ recovers the properties of admissibility and universality of $V_q$ as soon as $\qint$ is sufficiently high w.r.t. $q$ and $N$. The following is a particular case of~\cref{prop:qint_gen} in~\cref{subsec:proof_qint}.

\begin{proposition} \label{prop:qint}
Consider any $q, N \geq 2$ and any $\epsilon, \tau, R > 0$. Assume $\nu > 4$. \\
\textbf{(i)} For $\qint \geq \Omega(\tau q^8)$, with probability greater than $1-\tau^{-1}$, $\Hat{V}_q$ satisfies~\cref{ass:admissibility} with $\Hat{\kappa} \leq \kappa + 1$. \\
\textbf{(ii)} For $\qint \geq \Omega(\epsilon^{-2} N^2 (q \log(\v A \v_2 r_0 + R) + \tau)) $, with probability greater than $1 - e^{-\tau}$, for any $v \in L^2(\Hat{V}_q)$ s.t. $\v v \v_{L^2} \leq R$ and any time $t \in \tset$:
$
    \lambda_{\min}(\Hat{\mathbb{K}} ( (z^i_t)_i ) ) \geq \lambda_{\min}(\mathbb{K} ( (z^i_t)_i ) ) - \epsilon
$,
where $(z^i)_i$ are the solutions to~\cref{rkhs_forward} and $\Hat{\mathbb{K}}$, $\mathbb{K}$ are the kernel matrices of $\Hat{k}$ and $k$ respectively.
 %   \end{enumerate}
\end{proposition}

\begin{proof}[Sketch of proof for (i)]

%\textbf{Proof of (i).}
First note that for $\nu > 4$, $\mu_q$ admits up to $8^{th}$-order finite moments and these can be bounded uniformly in $q$~\cite{review_multivariate}.
Let $\varphi$ be the feature map of~\cref{RFF}. Then for every $z \in \r^q$, $\v \varphi (z) \v \leq 1$ so that for every $v \in \Hat{V}_q$, $\v v \v_{\infty} \leq \v W \v \v \varphi \v_{\infty} \leq \v v \v_V$. 
For the differential $Dv$ we have for every $z \in \r^q$:
\begin{align*}
    D \varphi(z) = \frac{1}{\sqrt{\qint}} \left( \omega^j_i e^{-\imath \langle z, \omega^j \rangle} \right)_{i,j} \in \r^{\qint \times q} \,.
\end{align*}
Then, by the Bienayme-Chebyshev inequality,
$D \varphi(z)^* D \varphi(z) = \frac{1}{\qint} \sum_{j= 1}^{\qint} \omega^j (\omega^j)^\top$ converges in probability to $- k''(0) \Id_q$ as $\qint \rightarrow+\infty$. Thus, for $\alpha > 0$ and $\qint$ sufficiently high w.r.t. $q$, $\alpha$ and $\tau$, $\v Dv \v_{2,\infty} = \v W D\varphi \v_{2, \infty} \leq \sqrt{- k''(0) +\alpha} \v v \v_{\Hat{V}_q}$, with probability greater than $1-\tau^{-1}$. 
The same idea applies to bound $\v D^2 v \v_{2, \infty}$ and the result follows using that $\kappa$ is given by~\cref{kappa}. 

\end{proof}

\begin{comment}

\textbf{Proof of (ii). }
For $t \in \tset$, we consider $(z_t^i)_i$ the solutions of of~\cref{rkhs_forward} for some control parameter $v \in L^2(\tset, \Hat{V}_q)$ and we introduce the kernel matrices:
\begin{align*}
    \Hat{\mathbb{K}}_t  = ( \Hat{K}_q (z_t^i, z_t^j) )_{1 \leq i,j \leq N},  \:
    \mathbb{K}_t = ( K_q(z_t^i, z_t^j) )_{1 \leq i,j \leq N}.
\end{align*}
Using the first point, we know that if $\v v \v_{L^2} \leq R$, then
$\v z^i_t \v \leq \v A x^i \v + (\kappa+1)R$.
Then, using Theorem 1 in~\cite{sriperumbudur2015optimal}, we have for every indices $i,j$ and every $t \in \tset$:
\begin{align*}
    \mathbb{P} \Big( | \Hat{k}(z^i_t, z^j_t) - k(\v z^i_t - z^j_t \v) | \geq \frac{h(q,R) + \sqrt{2\tau}}{\sqrt{\qint}} \Big) \leq e^{-\tau} ,
\end{align*}
with $h(q,R) = \mathcal{O}(\sqrt{q \log(\v A \v_2 r_0 + R)})$. Therefore, choosing $\qint \geq \Omega \left( \epsilon^{-2} N^2 ( q \log(\v A \v_2 r_0 + R) + \tau ) \right)$, we have with probability greater than $1 - e^{- \tau}$, $\lambda_{\min}(\Hat{\mathbb{K}}_t) \geq \lambda_{\min} \left( \mathbb{K}_t \right) - \epsilon$, for any $t \in \tset$ .

\end{comment}

%Finally, combining the results of~\cref{prop:embedding} and~\cref{prop:qint} we obtain a global convergence theorem for ResNets of finite width.

Finally, combining~\cref{prop:embedding} and~\cref{prop:qint}, we obtain a global convergence result. \Cref{thm:global_convergence} states convergence, with high probability over a choice of features, of GD towards a zero-training-loss optimum for infinitely deep ResNets of finite width.

\begin{theorem}[Global convergence] \label{thm:global_convergence}

Assume $\nu > 4$ and let $(x^i,y^i) \in ( \r^{d} \times \r^{d'} )^N$  be a compactly supported dataset with input data separation $\delta > 0$. There exist matrices $A \in \r^{q \times d}$ and $B \in \r^{d' \times q}$ s.t. for any $\tau > 0$, with probability at least $1 - \tau^{-1}$ w.r.t. the choices of features, GD initialized at $v^0=0$ converges towards a zero training loss optimum in the training of the \modelname model of~\cref{def:rkhs_resnet} with the feature map $\varphi$ of~\cref{RFF} as soon as $\qint \geq \Omega ( \tau (q^8 + q N^2 \log(\v A \v_2))$.
%, for some absolute constant $C$.
\end{theorem}

\begin{proof}

Consider $R = 1 $. By~\cref{prop:embedding},  we can have $A \in \r^{q\times d}$,  $B \in \r^{d' \times q}$ so that in~\cref{rkhs_init_cond}:
\begin{align*}
    \frac{ 8 \sqrt{2} \sigma_{\max}(B^\top) \sqrt{N \Lambda L(0)}  e^{3 (\kappa+1) }}{\sigma_{\min}(B^\top)^2 \lambda ( \sigma_{\min}(A)^{-1} \delta^{-1} e^{(\kappa+1)} ) } \leq 1  ,
\end{align*}
for $\kappa$, $\lambda$ and $\Lambda$ associated with $k$. Also, by the proof of~\cref{prop:embedding} we can have:
$
    \lambda ( \sigma_{\min}(A)^{-1} \delta^{-1} e^{(\kappa+1)} ) \geq 1/2 .
$
Taking $\epsilon = 1/4$ in~\cref{prop:qint}, the condition in~\cref{rkhs_init_cond} is satisfied by $\Hat{V}_q$ with probability greater than $1 - \tau^{-1}$ as soon as $\qint \geq \Omega(\tau q^8 + \tau q N^2 \log(1 + \v A \v_2 r_0 ) )$.
\end{proof}

%\Cref{thm:global_convergence} concludes that having a networks width $\qint \geq \Omega( \tau ( N^{32} + \delta^{-32} \log(N)^{32} )  )$
%(i.e. $ r \geq poly(\tau, N, \delta^{-1})$) implies convergence of GD towards a global optimum with probability greater than $1-\tau^{-1}$.
%Note that in order to obtain this result, the choice of the matrices $A_q, B_q$ in~\cref{prop:embedding} is not restrictive and the dependency of $q$ and $\qint$ w.r.t. $N$ and $\delta$ might be improved for other well-chosen embedding matrices. This is discussed at the end of~\cref{subsec:proof_q}. \todo{Reference to discussion on the choice of $A$ and $B$}

\section{Conclusion} \label{sec:conclusion}

We have identified a relevant infinite width limit (\modelname) for a particular model of ResNet. We showed that GD converges linearly when training this model and that a network's width polynomial w.r.t. to the size of the dataset is sufficient to maintain this property. 
A natural extension of our result is to study the convergence of GD when also training the hidden layers of the residuals. A first step towards this general case consists in studying the corresponding mean field model where the residuals are parameterized by density distributions over the neurons~\cite{chizat2018global,mei2018mean,mei2019mean,javanmard_analysis_2020,lu2020mean,fang2021modeling} for each residual blocks. Interestingly, such a parametrization of the residual blocks is still linear in this measure and thus fits into our framework of linear in parameters. However, it would require a finer mathematical analysis to obtain similar results.

{\bf Potential Negative Societal Impacts.} Our work aims at improving the theoretical and practical understanding of deep networks and therefore we do not expect a direct negative impact.

\vspace{-0.5em}
\section*{Acknowledgements}
\vspace{-0.7em}

The work of Gabriel Peyré was supported by the French government under management of Agence Nationale de la Recherche as part of the ”Investissements d’avenir” program, reference ANR19- P3IA-0001 (PRAIRIE 3IA Institute) and by the European Research Council (ERC project NORIA). This work was performed using HPC resources from GENCI-IDRIS (Grant 2022-[AD011013400]).

\clearpage

\bibliographystyle{siam} 
\bibliography{references}

\appendix

\clearpage

\section{Numerical experiments} \label{sec:expe}

The goal of this section is to quantify how much (in addition to interpolating the training dataset) our model is able to generalize on the test dataset. 
This is also useful to compare the performances of our model with those of standard ResNet architectures (which integrate batch normalization and training of the hidden layers).
We implemented our model in Pytorch~\cite{paszke2017automatic} and trained it on image datasets for classification tasks. Source code is available at \url{https://github.com/rbarboni/FlowResNets}.

Experiments were conducted using a private infrastructure, which has a carbon efficiency of 0.05~kgCO$_2$eq/kWh. A cumulative of (at most) 1000 hours of computation was performed on hardware of type Tesla V100-PCIE-16GB (TDP of 300W). Total emissions are estimated to be 15 kgCO$_2$eq (or 60km in an average car) of which 0 percents were directly offset. \\
Estimations were conducted using the \href{https://mlco2.github.io/impact#compute}{MachineLearning Impact calculator} presented in \cite{lacoste2019quantifying}.

\myparagraph{Computational setup for classification tasks.}
In the context of classification tasks, we use a cross entropy loss in place of the least square loss of \cref{empirical_risk}. For a problem with $K$ classes, the output dimension of the model is $d' = K$ and targets $y \in \r^K$ are one-hot vector encoding the target classes. 
For a batch of $N$ predictions $(z^i)_{\iset}$ and targets $(y^i)_{\iset}$ in $\r^K$ the \emph{Cross Entropy} loss is defined as:
\begin{align*}
    \mathtt{CrossEntropy}((z^i)_i, (y^i)_i) \eqdef \frac{1}{N} \sum_{i=1}^N \ell(z^i, y^i) ,
\end{align*}
where $\ell$ is the \emph{Binary Entropy} defined for one prediction $z$ and one target $y \in \r^K$ by:
\begin{align*}
    \ell(z,y) \eqdef \frac{\sum_{j=1}^K y^i_j e^{z^i_j} }{\sum_{j=1}^K e^{z^i_j} } .
\end{align*}
Then for a model $F$ depending on the parameters $W$ and a training batch $(x^i, y^i)_{\iset}$ we define the empirical risk:
\begin{align*}
    L(W) \eqdef \mathtt{CrossEntropy}((F(W, x^i)_i, (y^i)_i) ,
\end{align*}
and train the model by \emph{Stochastic Gradient Descent (SGD)} on $W$. Finally, the performance of the model is assessed by the \emph{Top-1 error rate} on a test dataset.

Note that, as explained in~\cref{PL_losses}, the result of \cref{thm:global_convergence} can be extended to this cross entropy loss. Indeed, $\ell$ satisfies a functional inequality similar to the Polyak-Lojasiewicz inequality. Assuming without loss of generality that $y = e_1$ is the indicator of class $1$, one has:
\begin{align*}
    \nabla_{z_1} \ell (z, y) & = e^{- \ell(z,y)} -1,
    %\forall j \geq 2, \ \nabla_{z_j} \ell (z,y) &= e^{- \ell(z,y)} .
\end{align*}
Then by convexity of exponential, when $\ell(z,y) \leq 1$:
\begin{align*}
    \v \nabla_z \ell(z,y) \v^2 & \geq (1 - e^{-\ell(z,y)})^2 \geq (1 - e^{-1})^2 \ell(z,y)^2 .
\end{align*}
Note however that \cref{thm:global_convergence} is only valid for full batch gradient descent. We leave its extension to SGD for future works.

\subsection{Experiments on MNIST} \label{sec:setup}
We implemented the model of~\cref{def:rkhs_resnet} on Pytorch using the $\mathtt{torchdiffeq}$ package~\cite{chen2018neural} and performed experiments on the MNIST dataset.

\myparagraph{Implementation using $\mathtt{torchdiffeq}$.}
The model of~\cref{def:rkhs_resnet} is implemented as a succession of convolutional layers. Given some number of layers $L$ the trained parameters consist of convolution matrices $W_k \in \r^{C \times C_{int} \times 3 \times 3}$ for $k \in \llbracket 0, L \rrbracket$, with $C$ the number of channels of the input image and $C_{int}$ some number of channels for the hidden layers. The control parameter $v$ is defined at discrete time steps $\lbrace k/L \rbrace_{0 \leq k \leq L}$ by:
\begin{align*}
    v_{k/L}(x) = W_k \star \mathtt{ReLU} ( U \star x) ,
\end{align*}
where $U \in \r^{C_{int} \times C \times 3 \times 3}$ is a fixed and untrained convolution matrix. We refer to this setting as a ResNet with RKHS residuals. On the other hand, we refer to the setting where $U$ is replaced at each layer by trained convolution matrices $U_k$ as ResNet with $\emph{Single Hidden Layer (SHL)}$ residuals.

\begin{remark}
By analogy with the definition of RKHSs generated by random features (\cref{RFF}), the ratio between the number of features and the dimension is here:
\begin{align*}
    \frac{r}{q} = \frac{C_{int}}{C} .
\end{align*}
\end{remark}

For any $t \in \tset$, $v_t$ is defined by affine interpolation:
\begin{align*}
    v_t(x) & \eqdef v_{k/L}(x) + (tL - k) \left( v_{(k+1)/L}(x) - v_{k/L}(x) \right) \\
    & = (W_k + (tL - k) (W_{k+1}-W_k)) \star \sigma( U \star x) ,
\end{align*}
with $k = \lfloor tL \rfloor$. The forward method consists in integrating the ODE of~\cref{rkhs_forward} with control parameter $v$ using the $\mathtt{torchdiffeq.odeint}$ method~\cite{chen2018neural}. For some input $z_0$ define:
\begin{align*}
    z_1((W_k), z_0) \eqdef \mathtt{torchdiffeq.odeint}(v, z_0, \tset) ,
\end{align*}
then for an image input $x$ the model's output is given by:
\begin{align*}
    F((W_k), x) = \mathtt{B}(z_1((W_k), \mathtt{A}(x))) ,
\end{align*}
where $\mathtt{A}$ and $\mathtt{B}$ are small convolutional networks, fixed during the training of $F$. These networks play the same role as the matrices $A$ and $B$ in~\cref{def:rkhs_resnet}, that is they are used for the purpose of adjusting the data dimension.

\myparagraph{Hyperparameter tuning.}
Several choices of hyperparameters can affect the performances of the model.
\begin{itemize}
    \item The convolution matrix $U$: as detailed in~\cref{sec:convergence}, the way the weights of $U$ are sampled determines to which RKHS $V$ belongs the control parameter $v$. For the sake of simplicity we choose to sample the coefficients of $U$ as i.i.d. Gaussians.
    \item The initialization of $(W_k)$: the weights of the convolution matrices $W_k$ are initialized to $0$. For an input image $x$ the output is given at initialization by $F(0,x) = \mathtt{B}( \mathtt{A} ( x ))$.
    \item The integration method: $\mathtt{torchdiffeq.odeint}$ allows the user to choose an integration method. We observed an \emph{explicit midpoint} method to
offer a good trade-off between performance and numerical stability w.r.t. other fixed-steps methods such as \emph{explicit Euler} or \emph{RK4}.
    \item The number of layers $L$: we tested our model for $L \in \lbrace 5, 10, 20 \rbrace$. This parameter controls the total number of parameters of the model.
    \item The networks $\mathtt{A}$ and $\mathtt{B}$: their choice defines the dimension of space in which the forward ODE~\cref{rkhs_forward} is integrated, which is expected to have an important impact on the performances of the model (c.f.~\cref{sec:convergence}). Moreover, as the parameters $(W_k)$ are initialized at $0$, the performances of the model before training are those of the concatenation $\mathtt{B} \circ \mathtt{A}$. Without training, the classification error of $\mathtt{B} \circ \mathtt{A}$ is of $90\%$ while with enough training, it can be as good as $2\%$. We tested our model with different levels of training of $\mathtt{B} \circ \mathtt{A}$. 
\end{itemize}

\myparagraph{Results.}
\Cref{fig:32channels} shows the evolution of the performances of \modelnames while trained on the MNIST dataset. The decay of the Empirical Risk is directly related to the decay of the classification error. Without pretraining $\mathtt{A}$ and $\mathtt{B}$, our model already achieves up to $98\%$ accuracy
on the test set. When $\mathtt{A}$ and $\mathtt{B}$ are pretrained \modelname still improves on the starting accuracy: in this setting more than $99\%$ accuracy is reached. Most importantly,~\cref{fig:32channels} shows that not training the hidden layers inside residual blocks does not significantly hinders the performances in classification. Indeed, comparing the performances of ResNets with RKHS residuals and SHL residuals one observes a $1\%$ accuracy drop when training \modelname from scratch (\cref{fig:32channels_no_pretrain}) and $0.5\%$ accuracy when networks $\mathtt{A}$ and $\mathtt{B}$ are pretrained (\cref{fig:32channels_pretrain}).

Finally we showcase the relevance of the analysis of~\cref{sec:convergence} by training our model with a varying number of input channels in~\cref{fig:8channels}. We observe a significant drop in convergence of the empirical risk with $4$ channels compared with $8$ and $32$ channels. Non-convergence of the empirical risk also implies poorer performances in generalization. Such results are coherent with the convergence condition of~\cref{rkhs_init_cond}: augmenting the data dimension allows to have global convergence when the loss at initialization is too high.

\subsection{Experiments on CIFAR10}

We performed experiments on the CIFAR10 dataset, using an architecture inspired from ResNet18~\cite{he2016deep}.

\myparagraph{Implementation.}
Our architecture relies on the ResNet18 architecture~\cite{he2016deep} but residual blocks are changed and simplified (by removing the final non-linearity and the batch-normalization) to match the definition of \modelname
(\cref{def:rkhs_resnet}). Each residual block consists in the composition of a convolution $U$, a ReLU non-linearity and a convolution $W$. More precisely, for an input image $x$, the output of the $k^{th}$ layer reads:
\begin{align*}
    \mathcal{F}_k(x) = x + W_k \star \mathtt{ReLU}( U_k \star x ) ,
\end{align*}
where $U_k  \in \r^{C_{int} \times C \times 3 \times 3}$, $W_k  \in \r^{C \times C_{int} \times 3 \times 3}$ are convolution matrices, $C$ is the number of channels of the input image and $C_{int}$ is the number of channels of the hidden layer. When both convolutions $W_k$ and $U_k$ are trained, we refer to these residuals as \emph{Single Hidden Layer (SHL) residuals}. In the framework of \modelname, all convolutions $U_k$ are fixed and set to the same convolution $U$. We refer to it as \emph{RKHS residuals}.

Finally, ResNet18 consists of 4 blocks each containing 2 residual layers. We keep 2 of our residuals in the first, second and fourth block but stack an arbitrary number $D$ of residual layers in the third block. Thereby we refer to this third block as the NODE block, which performs the integration of~\cref{rkhs_forward}.

Note that compared to the residuals in the original ResNet18 architecture, batch-normalization at input and output of the residuals as well as ReLU non-linearities are removed. Moreover, in order to reproduce the framework of \emph{Random Fourier Features} (\cref{RFF}), the weights of $U$ are sampled as i.i.d. gaussians and rescaled by a $C_{int}^{-1/2}$ factor. Finally, the weights of the convolutions $W_k$ are initialized at $0$. Such an initialization corresponds in many ways to the one proposed in~\cite{zhang2018fixup}.

\myparagraph{Results.}
\cref{fig:CIFAR} reports the training of \modelname on the CIFAR10 dataset. \Cref{fig:CIFAR_rkhs} shows the training of \modelname (RKHS residuals) and is to be compared with~\cref{fig:CIFAR_shl} which shows the training of the same model but with trained hidden layers (SHL residuals). Our experiments show that similar performances can be achieved: both ResNets achieve up to $88\%$ accuracy on the test dataset. As a comparison, the ResNet18 original architecture can be trained to achieve up to $94\%$ accuracy.

Finally,~\cref{fig:CIFAR} also compares the performances of the model depending on the number of layers inside the NODE block. One observes significantly different behavior when there is no NODE ($1$ layer) and one there is ($10$ and $20$ layers): more layers are related to better performances both on the train dataset and on the test dataset and both when hidden layers are trained or not. However, one sees that the improvement related to adding more layers is limited: performances with $10$ and $20$ layers are very similar and a NODE block with $1$ layers already achieves $82\%$ accuracy RKHS residuals and $84\%$ accuracy with SHL residuals.

\begin{figure}
    \centering
    \begin{subfigure}{\textwidth}
        \centering
        \includegraphics[scale=1]{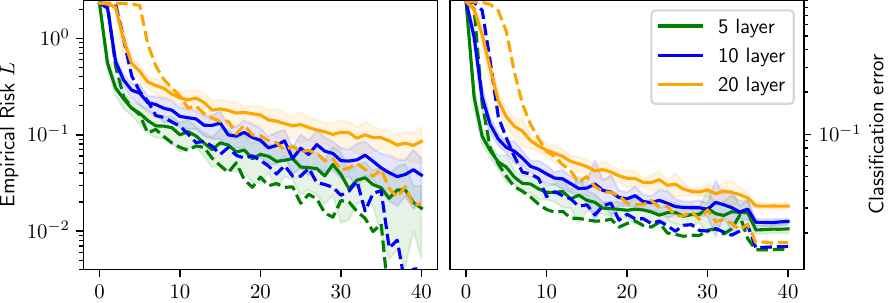}
        \vspace{-1.5em}
        \caption{Without pretraining (starting with $10\%$ accuracy)} \label{fig:32channels_no_pretrain}
    \end{subfigure}
    \begin{subfigure}{\textwidth}
        \centering
        \includegraphics[scale=1]{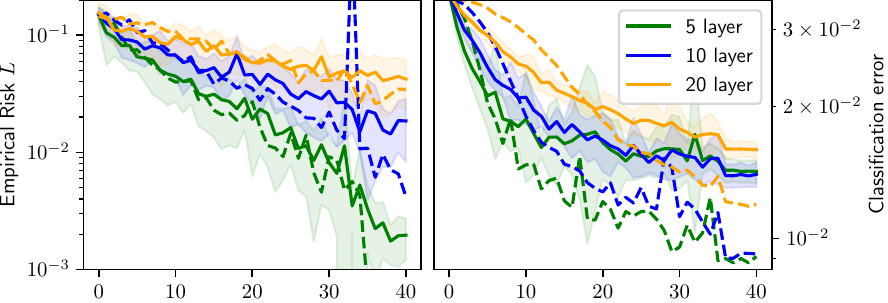}
        \vspace{-1.5em}
        \caption{With pretraining (starting with $97\%$ accuracy)} \label{fig:32channels_pretrain}
        %\vspace{-1.9em}
    \end{subfigure}
    \vspace{-0.5em}
    \caption{Performances of NODE with $32$ channels while trained on MNIST with SGD. Left column reports evolution of the empirical risk and right column reports evolution of classification error,  both for ResNets with RKHS residuals (plain) and SHL residuals (dashed). The $x$-axis is the number of pass through the dataset. Experiments are performed with different levels of pretraining of $\mathtt{A}$ and $\mathtt{B}$, corresponding to different starting accuracy ((a)-(b)), and with different number of layers. Learning rate and batch size are fixed, learning rate is divided by $10$ after $35$ iterations. Plots are average over 20 runs, lines are means and, for RKHS residuals, colored areas are mean $\pm$ one standard deviation.}
    \label{fig:32channels}
\end{figure}

\begin{figure}
    \centering
    \includegraphics[scale=1]{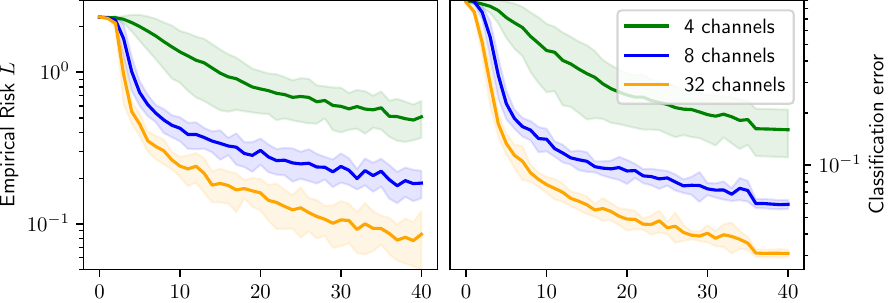}
    \vspace{-0.5em}
    \caption{Training of \modelname on MNIST with $20$ layers, $4$, $8$  and $32$ input channels $C$ and without pretraining. The $x$-axis is the number of pass through the dataset. The rate $C_{int} / C = 1$ is the same for each model. Learning rate and batch size are fixed, learning rate is divided by $10$ after $35$ iterations. Plots are average over 20 runs, lines are means and colored areas are mean $\pm$ one standard deviation.}
    \label{fig:8channels}
\end{figure}

\begin{figure}
    \begin{subfigure}{\textwidth}
        \centering
        \includegraphics[scale=1]{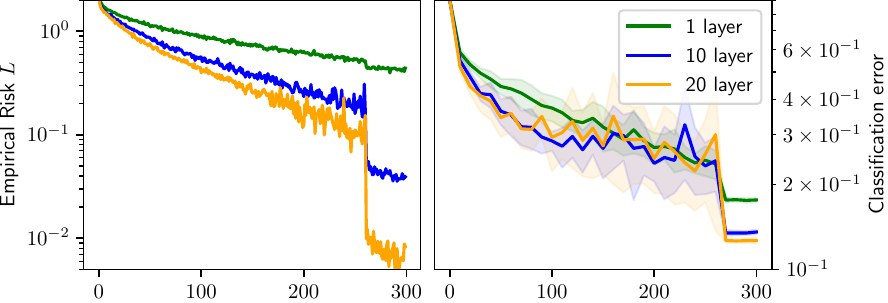}
        \vspace{-1.5em}
        \caption{Fixed hidden layers (RKHS)}
        \label{fig:CIFAR_rkhs}
        %\vspace{-1.9em}
    \end{subfigure}
    \begin{subfigure}{\textwidth} 
        \centering
        \includegraphics[scale=1]{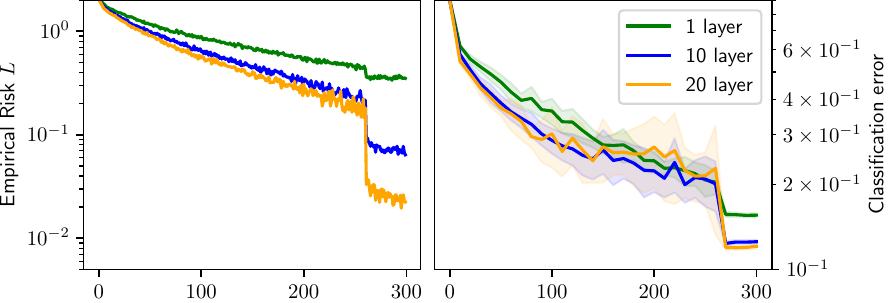}
        \vspace{-1.5em}
        \caption{Trained hidden layers (SHL)}
        \label{fig:CIFAR_shl}
        %\vspace{-1.9em}
    \end{subfigure}
    \vspace{-0.5em}
    \caption{Performances of \modelname while trained on CIFAR10 with SGD (256 images per batch). Left column reports evolution of the empirical risk on the train set and right column reports the classification error on the test set. The $x$-axis is the number of pass through the dataset. Learning rate and batch size are fixed, learning rate is divided by $10$ after $260$ iterations. Plots are average over 20 runs, lines are means and colored areas are mean $\pm$ one standard deviation.}
    \label{fig:CIFAR}
\end{figure}

\section{Proofs of~\cref{sec:overparameterized}} \label{sec:proof_overparameterized}

We give a proof of~\cref{thm:convergence_gen}. This essentially follows the proof given in~\cite{liu_loss_2021}.

\begin{proof}[Proof of~\cref{thm:convergence_gen}]

Assume the loss $L$ satisfies~\cref{def:PL_gen} with $M$ and $m$ and that~\cref{init_cond_gen} is satisfied at initialization $v^0 \in \r^m$. The proof proceeds by induction over the gradient step $k$

Assume the convergence rate and the regularization bound of~\cref{convergence_rate_gen} are satisfied for every $l \leq k$. Then at step $k+1$:
\begin{align*}
   \v v^{k+1} - v^0 \v & = \v \eta \sum_{l=0}^k \nabla L(v^l) \v
   \leq \eta \sum_{l=0}^k \v \nabla L(v^l) \v \\
   & \leq \eta \sum_{l=0}^k \sqrt{2 M(\v v^l \v) L(v^l)} . \\
\end{align*}
Using the induction hypothesis and setting $\mu = m(\v v^0 \v + R)$ we have:
\begin{align*}
   \v v^{k+1} - v^0 \v & \leq \eta \sqrt{2 M(\v v^0 \v + R) L(v^0) } \sum_{l=0}^k (1 - \eta \mu)^{-l/2} \\
   & \leq \eta \sqrt{2 M(\v v^0 \v + R) L(v^0) } (1 - \sqrt{1 - \eta \mu })^{-1} \\
   & \leq \frac{2}{\mu} \sqrt{2 M(\v v^0 \v + R) L(v^0) } \\
   & \leq R ,
\end{align*}
where the last inequality is~\cref{init_cond_gen}. We thus recovered the regularization bound of~\cref{convergence_rate_gen} at step $k+1$.

Moreover, because $v^{k+1}$ is located in $B(v^0,R)$ we have thanks to the smoothness assumption:
\begin{align*}
    L(v^{k+1}) & \leq L(v^k) - \eta \v \nabla L(v^k) \v^2 + \eta^2 \frac{\beta}{2} \v \nabla L(v) \v^2 \\
    & \leq L(v^k) - \frac{\eta}{2} \v \nabla L(v^k) \v^2 ,
\end{align*}
because $\eta \leq \beta^{-1}$. Thus using the lower bound in the PL inequality~\cref{PL_ineq_gen}:
\begin{align*}
    L(v^{k+1}) \leq L(v^k) ( 1 - m(\v v^0 \v + R) \eta ) ,
\end{align*}
which gives the convergence rate of~\cref{convergence_rate_gen} at step $k+1$ by induction on $k$.

\end{proof}

\section{Proofs of~\cref{sec:main}}

\subsection{About the definition of \modelname}

Before deriving proofs for the properties of our \modelname model, it is interesting to study carefully the well-posedness of~\cref{def:rkhs_resnet}. Indeed, because the control parameter $v$ is only integrable in time and not continuous, the Cauchy-Lipschitz
theorem does not ensure that there exist solutions to~\cref{rkhs_forward}. Instead we rely on a weaker notion of solution and use a result from Carathéodory (Section I.5 in~\cite{hale_ordinary_2009}).

\begin{proposition} \label{prop:well_posedness}
Let $V$ be some RKHS satisfying~\cref{ass:admissibility} and $v \in \Lspace$ be some control parameter. Then for every $x \in \r^d$ there exists a unique solution  $z$ of~\cref{rkhs_forward} in the weak sense of absolutely continuous functions. More precisely there exists a unique $z \in H^1(\tset, \r^q)$ such that for every $t \in \tset$:
\begin{align} \label{weak_forward}
    z_t = A x + \int_0^t v_s (z_s) \text{d}s \,.
\end{align}
\end{proposition}

\begin{proof}
The map $(t,z) \in \tset \times \r^q \mapsto v_t(z)$ is measurable and by~\cref{ass:admissibility} we have for every $t \in \tset$ and every $z \in \r^q$:
\begin{align*}
    \v v_t(z) \v \leq \kappa \v v_t \v_V\, ,
\end{align*}
whose upper-bound is integrable w.r.t. $t \in [0,1]$. Then, applying Theorem 5.1 of~\cite{hale_ordinary_2009} gives a unique absolutely continuous solution $z$ of~\cref{weak_forward}. Applying~\cref{ass:admissibility} once again, we have that $\Dot{z}$ is square integrable and thus $z$ is in $H^1$.
\end{proof}

In the paper, every equality implying derivatives has to be understood in the sense of weak derivatives of $H^1$ functions. In particular, this notion allows to perform integration by parts, which is used in the following proof of~\cref{prop:grad}.

\begin{proof}[Proof~\cref{prop:grad}]

Consider the optimization problem of minimizing the empirical risk of~\cref{empirical_risk} with $F$ the \modelname model of~\cref{def:rkhs_resnet} and a dataset $(x^i,y^i)_{\iset} \in (\r^d \times \r^{d'})^N$. Introducing for every index $i \in \llbracket 1,N \rrbracket$ the variables $z^i \in H^1(\tset, \r^q)$ solutions of~\cref{rkhs_forward}, this can be viewed as an optimisation problem over $((z^i)_i,v)$ under the constraint that~\cref{rkhs_forward} is satisfied:
\begin{align*}
    \underset{
    \substack{(z^i)_i \in H^1(\r^q)^N \\ v \in L^2(V)}
    }{\min} & \frac{1}{2N} \sum_{i=1}^N \v B z^i_1 - y^i \v^2 \\
    \text{with} \ \forall i \in \llbracket 1,N \rrbracket, & 
    \left\{
        \begin{array}{rcl}
            \Dot{z}^i_t & = & v_t(z^i_t) \ \forall t \in \tset  \\
            z^i_0 & = & A x^i . 
        \end{array}
    \right.
\end{align*}

Introducing the adjoint variables $(p^i)_i \in H^1(\r^q)^N$, the Lagrangian of the optimization problem is defined as:
\begin{align*}
    \mathcal{L}((z^i),(P^i),v) & \eqdef \sum_{i=1}^N \big( \frac{1}{2N} \v B z^i_1 - y^i \v + \int_0^1 \langle p^i_t, \Dot{z}^i_t - v_t(z^i_t) \rangle \text{d}t \big) \\
    & = \sum_{i=1}^N \big( \frac{1}{2N} \v B z^i_1 - y^i \v + \left[\langle p^i_t, z^i_t \rangle \right]_0^1 - \int_0^1 \langle \Dot{p}^i_t, z^i_t \rangle \text{d}t - \int_0^t \langle p^i_t , v_t(z^i_t) \rangle \text{d}t \big) ,
\end{align*}
where the second equality is established by integration by parts. Therefore, the condition for optimality over $z^i$ is equivalent to~\cref{rkhs_backward}. For every index $i$:
\begin{align*}
    \nabla_{z^i} \mathcal{L} = 0
    \Leftrightarrow  
    \left\{
        \begin{array}{rcl}
            \Dot{p}^i_t & = & - Dv_t(z^i_t)p^i_t  \\
            p^i_1 & = & - \frac{1}{N} B^\top (B z^i_1 - y^i) ,
        \end{array}
    \right.
\end{align*}
which has to be understand in the sense of weak solutions in $H^1$.

The gradient of $L$ is obtained by differentiating over the $v$ variable. Denoting $\delta_z^p$ the linear form $v \mapsto \langle v(z), p \rangle$, we have:
\begin{align*}
    \nabla L(v) & = \nabla_v \mathcal{L}((z^i),(p^i),v) \\
    & = - \sum_{i=1}^N K * \delta_{z^i}^{p^i} \\
    & = - \sum_{i=1}^N K(.,z^i)p^i ,
\end{align*}
with $K$ the kernel function of the RKHS $V$ and $K* : V^* \rightarrow V$ the associated isometry\footnote{The notation $K*$ reminds of convolution which is the case when the kernel is translation invariant.}. 

\end{proof}

\subsection{Proof of~\cref{prop:PL_rkhs}}\label{subsec:proof_PL_rkhs}

We prove here that for any given dataset $(x^i,y^i)_{\iset} \in (\r^d \times \r^{d'})^N$, the empirical risk $L$ associated with the \modelname model satisfies a (local) Polyak-Lojasiewicz property. As stated in \cref{prop:PL_rkhs}. The proof uses~\cref{ass:admissibility} to derive estimates on the solutions of~\cref{rkhs_forward} and~\cref{rkhs_backward}, which we give in the following lemma:

\begin{lemma} \label{lemma1}
Let $V$ satisfy~\cref{ass:admissibility} with constant $\kappa$ and let $v \in L^2(\tset,V)$ be some control parameter.

\textbf{(i)} Let $(z^i)_{\iset}$ be the solutions of~\cref{rkhs_forward} for some data inputs $(x^i)_{\iset} \in (\r^d)^N$. Then for every indices $i,j \in \llbracket 1, N \rrbracket$ and every time $t \in \tset$:
\begin{align}
    \v z^i - z^j \v \geq \sigma_{\min}(A) e^{- \kappa \v v \v_{L^2}} \v x^i - x^j \v \,.
\end{align}

\textbf{(ii)} Let $(p^i)_{\iset}$ be the solutions of~\cref{rkhs_backward} associated with $(z^i)_{\iset}$ with objective outputs $(y^i)_{\iset} \in (\r^{d'})^N$. Then for every $i \in \llbracket 1, N \rrbracket$ and every time $t \in \tset$:
\begin{align*}
    \frac{\sigma_{\min}(B^\top)}{N} e^{-\kappa \v v \v_L^2} \v B z^i_1 - y^i \v \ \leq \ 
    \v p^i_t \v \ \leq \ \frac{\sigma_{\max}(B^\top)}{N} e^{\kappa \v v \v_L^2} \v B z^i_1 - y^i \v\,.
\end{align*}
\end{lemma}

\begin{proof}[Proof of~\cref{lemma1}]

\textbf{Proof of (i)} Let $i,j \in \llbracket 1,N \rrbracket$. Assume by contradiction that for some time $t \in \tset$ we have:
\begin{align*}
        \v z^i_t - z^j_t \v < e^{- \kappa \v v \v_{L^2}} \v z^i_0 - z^j_0 \v .
\end{align*}
Then because $z^i$ and $z^j$ are absolutely continuous, $\v z^i - z^j \v^2$ is absolutely continuous and for any time $s \in \tset$:
\begin{align*}
    \v z^i_s - z^j_s \v^2 & = \v z^i_t - z^j_t \v^2 + 2 \int_t^s \langle v_r(z^i_r) - v_r(z^j_r), z^i_r - z^j_r \rangle \text{d}r \\
    & \leq \v z^i_t - z^j_t \v^2 + 2 \int_t^s \kappa \v v_r \v_V \v z^i_r - z^j_r \v^2 \text{d}r ,
\end{align*}
where the inequality follows from $\v Dv_r \v_{2,\infty} \leq \kappa \v v_r \v_V$. Applying Grönwall's lemma, we have:
\begin{align*}
    \v z^i_s - z^j_s \v^2 \leq \v z^i_t - z^j_t \v^2 e^{ 2 \kappa \v v \v_{L^2}} ,
\end{align*}
and by setting $s=0$:
\begin{align*}
    \v z^i_0 - z^j_0 \v^2 \leq \v z^i_t - z^j_t \v^2 e^{ 2 \kappa \v v \v_{L^2}} < \v z^i_0 - z^j_0 \v ,
\end{align*}
which is a contradiction. Therefore for any time $t \in \tset$:
\begin{align*}
    \v z^i_t - z^j_t \v \geq e^{- \kappa \v v \v_{L^2}} \v z^i_0 - z^j_0 \v ,
\end{align*}
and the result follows by considering the initial condition $z^i_0 = A x^i$.

\textbf{Proof of (ii)} Let $i \in \llbracket 1, N \rrbracket$ be any index and let $p^i$ be the solution of~\cref{rkhs_backward} with initial condition $p^i_1 = - \frac{1}{N} B^\top (B z^i_1 - y^i)$. Then because $p^i$ is absolutely continuous, $\v p^i \v$ is absolutely continuous and for any time $t \leq s \in \tset$:
\begin{align*}
    \v p^i_t \v^2 = \v p^i_1 \v^2 - 2 \int_1^t \langle Dv_s(z^i_s) p^i_s, p^i_s \rangle \text{d}s ,
\end{align*} 
so that using~\cref{ass:admissibility} we have:
\begin{align*}
    \v p^i_s \v^2 \leq \v p^i_t \v^2 + 2 \int_t^s \kappa \v v_r \v_V \v p^i_r \v^2 \text{d}r \,. \\
\end{align*}
Using Grönwall's lemma in the first inequality and setting $s=0$ we have:
\begin{align*}
    \v p^i_1 \v^2 \leq \v p^i_t \v^2 e^{2 \kappa \v v \v_{L^2}} ,
\end{align*}
and proceeding by contradiction (such as in (i)) we have:
\begin{align*}
    \v p^i_1 \v^2 \geq \v p^i_t \v^2 e^{- 2 \kappa \v v \v_{L^2}} .
\end{align*}
The result follows by considering the initial condition on $p^i_1$.
\end{proof}

Provided those estimates on $z^i$ and $p^i$, it remains to use~\cref{ass:universality} in order to conclude.

\begin{proof}[Proof of~\cref{prop:PL_rkhs}]
Let $v \in \Lspace$ and consider the form of the gradient of $L$ given by~\cref{prop:grad} with $(z^i)_{\iset}$ the solutions of~\cref{rkhs_forward} and $(p^i)_{\iset}$ the solutions of~\cref{rkhs_backward}. Let $t \in \tset$, then by definition of the norm in RKHSs:
\begin{align*}
    \v \nabla L(v)_t \v_V^2 = \sum_{1 \leq i,j \leq N} (p^i_t)^\top K(z^i_t, z^j_t) p^j_t\, ,
\end{align*}
where we recall that $K$ is the kernel associated with $V$. Noting $p \eqdef (p^i_t) \in \r^{Nq}$, the vector of the stacked $(p^i_t)_{\iset}$, and $\mathbb{K}$ the kernel matrix associated with the family of points $(z^i_t)_i$, we have:
\begin{align*}
    \v \nabla L(v)_t \v_V^2 & = \langle p, \mathbb{K} p \rangle \,.
\end{align*}
Then by~\cref{ass:universality}, there exists a non-increasing function $\lambda$ and a constant $\Lambda$ such that:
\begin{align*}
    \lambda (\max_{1 \leq i,j \leq N} \v z^i_t - z^j_t \v^{-1} ) \v p \v^2 \ \leq \ \v \nabla L(v)_t \v_V^2 \ \leq \ \Lambda \v p \v^2 .
\end{align*}
Using (i) in~\cref{lemma1} we have:
\begin{align*}
    \lambda (\max_{1 \leq i,j \leq N} \v z^i_t - z^j_t \v^{-1}) \geq \lambda(\sigma_{\min}(A)^{-1} \delta^{-1} e^{\kappa \v v \v_{L^2}}) ,
\end{align*}
where $\delta \eqdef \min_{1 \leq i,j \leq N} \v x^i - x^j \v$ is the data separation. Finally the result follows by using (ii). More precisely:
\begin{align*}
    \v p \v^2 & = \sum_{i=1}^N \v p^i_t \v^2 \\
    & \leq \frac{\sigma_{\max}(B^\top)^2}{N^2} e^{2 \kappa \v v \v_{L^2}} \sum_{i=1}^N \v B z^i_1 - y^i \v^2 \\
    & = 2 \frac{\sigma_{\max}(B^\top)^2}{N} e^{2 \kappa \v v \v_{L^2}} L(v) ,
\end{align*}
and in the same manner:
\begin{align*}
     \v p \v^2 \geq 2 \frac{\sigma_{\min}(B^\top)^2}{N} e^{- 2 \kappa \v v \v_{L^2}} L(v) .
\end{align*}

\end{proof}

\subsection{Proof of~\cref{thm:main}} \label{subsec:proof_main}

\Cref{thm:main} is a direct consequence of~\cref{prop:PL_rkhs}. In order to apply~\cref{thm:convergence_gen}, it suffices to show that $L$ satisfies some smoothness assumption as defined in~\cref{def:smoothness}:

\begin{property}[Smoothness of $L$] \label{prop:smoothness}

Let $V$ be some RKHS satisfying~\cref{ass:admissibility}. Let $L$ be the empirical risk defined on $\Lspace$ and associated with the \modelname model. Then there exists a continuous function $\mathbf{C} : \r_+ \rightarrow \r_+^*$ such that for every $R \geq 0$ and every $v,\Bar{v} \in \Lspace$ with $\v v \v_{L^2}, \v \Bar{v} \v_{L^2} \leq R$:
\begin{align*}
    \v \nabla L(v) - \nabla L(\Bar{v}) \v_{L^2} \leq \mathbf{C}(R) \v v - \Bar{v} \v_{L^2}.
\end{align*}
\end{property}

We note $\kappa$ the constant associated with~\cref{ass:admissibility}. The proof of~\cref{prop:smoothness} relies on the following lemma:

\begin{lemma} \label{lemma2}
Let $v, \Bar{v} \in \Lspace$ be some control parameters and $R \geq 0$ be some radius such that $\v v \v_{L^2}, \v \Bar{v} \v_{L^2} \leq R$. Let $(x,y) \in \r^d \times \r^{d'}$ be some pair of data input / objective output.

\textbf{(i)} Let $z, \Bar{z}$ be solutions of~\cref{rkhs_forward} with parameter $v$ and $\Bar{v}$ respectively and with the same initial condition $Ax$, then for any $t \in \tset$:
\begin{align*}
    \v z_t - \Bar{z}_t \v \leq \kappa e^{\kappa R} \v v - \Bar{v} \v_{L^2} .
\end{align*}

\textbf{(ii)} Let $p, \Bar{p}$ be solutions of~\cref{rkhs_backward} with parameter $v$ and $\Bar{v}$ respectively and with initial condition $\frac{1}{N}B^\top (B z_1 - y )$ and $\frac{1}{N}B^\top (B \Bar{z}_1 - y )$, then for any $t \in \tset$:
\begin{multline*}
    \v p_t - \Bar{p}_t \v \leq \\
    \frac{\kappa e^{2 \kappa R} \v B \v_2}{N} \v v - \Bar{v} \v_{L^2} \big[ \v B \v_2 + \v B ( \Bar{z}_1 - y ) \v (1 + R e^{\kappa R}) \big]\, .
\end{multline*}
\end{lemma}

\begin{proof}[Proof of~\cref{lemma2}]

\textbf{Proof of (i)} For every time $t \in \tset$ we have:
\begin{align*}
    z_t - \Bar{z}_t & = \int_0^t \big( v_s(z_s) - \Bar{v}_s(\Bar{z}_s) \big) \text{d}s \\
    & = \int_0^t \big( v_s(z_s) - v_s (\Bar{z}_s) + v_s (\Bar{z}_s) - \Bar{v}_s(\Bar{z}_s) \big) \text{d}s\, ,
\end{align*}
and by triangle inequality:
\begin{align*}
    \v z_t - \Bar{z}_t \v & \leq \int_0^t \big( \v v_s(z_s) - v_s (\Bar{z}_s) \v + \v v_s (\Bar{z}_s) - \Bar{v}_s(\Bar{z}_s) \v \big) \text{d}s \\
    & \leq \int_0^t \kappa \v v_s \v_V \v \v z_s - \Bar{z}_s \v \text{d}s + \int_0^t \kappa \v v_s - \Bar{v}_s \v_V \text{d}s \,,
\end{align*}
where we used~\cref{ass:admissibility} in the second inequality. Therefore, by Grönwall's lemma:
\begin{align*}
    \v z_t - \Bar{z}_t \v & \leq \kappa e^{\kappa \v v \v_{L^2}} \int_0^t \v v_s - \Bar{v}_s \v_V \text{d}s \\
    & \leq \kappa e^{\kappa R} \v v - \Bar{v} \v_{L^2} \, .
\end{align*}

\textbf{Proof of (ii)} For any $t \in \tset$ we have:
\begin{align*}
    p_t - \Bar{p}_t & =  (p_1 - \Bar{p}_1) - \int_1^t \big( Dv_s(z_s)^\top p_s - D\Bar{v}_s (\Bar{z}_s)^\top \Bar{p}_s \big) \text{d}s \\
    & = (p_1 - \Bar{p}_1) - \int_1^t \left[  Dv_s(z_s)^\top ( p_s -\Bar{p}_s )
     + \big( Dv_s (z_s) - Dv_s( \Bar{z}_s) \big)^\top \Bar{p}_s
     + \big( D v_s( \Bar{z}_s) - D\Bar{v}_s(\Bar{z}_s) \big)^\top \Bar{p}_s \right] \text{d}s \,,
\end{align*}
and using the triangle inequality and~\cref{ass:admissibility}:
\begin{align*}
    \v p_t - \Bar{p}_t \v \leq \v p_1 - \Bar{p}_1 \v + \int_t^1 \left[ \kappa \v v_s \v_V \v p_s - \Bar{p}_s \v  + \kappa \v v_s \v_V \v z_s - \Bar{z}_s \v \v \Bar{p}_s \v +  \kappa \v v_s - \Bar{v}_s \v_V \v \Bar{p}_s \v \right] \text{d}s .
\end{align*}
Then, using Grönwall's lemma backward in time gives:
\begin{align*}
    \v p_t - \Bar{p}_t \v \leq \v p_1 - \Bar{p}_1 \v e^{\kappa \v v \v_{L^2}}
    + \kappa e^{\kappa \v v \v_{L^2}} \int_t^1  \v v_s - \Bar{v}_s \v_V \v \Bar{p}_s \v \text{d}s
    + \kappa e^{\kappa \v v \v_{L^2}} \int_t^1 \v v_s \v_V \v z_s - \Bar{z}_s \v \v \Bar{p}_s \v \text{d}s .
\end{align*}
On one hand, because of (i) we have for every $s \in \tset$:
\begin{align*}
    \v z_s - \Bar{z}_s \v & \leq \kappa e^{\kappa R} \v v - \Bar{v} \v_{L^2},
\end{align*}
and also:
\begin{align*}
    \v p_1 - \Bar{p}_1 \v & = \frac{1}{N} \v B^\top B ( z_1 - \Bar{z}_1 ) \v \\
    & \leq \frac{\v B \v_2^2}{N} \kappa e^{\kappa R} \v v - \Bar{v} \v_{L^2} .
\end{align*}
On the other hand, recalling (ii) of~\cref{lemma1}, for every $s \in \tset$:
\begin{align*}
    \v \Bar{p}_s \v \leq \frac{\sigma_{\max}(B^\top)}{N} e^{\kappa R} \v B z_1 - y \v .
\end{align*}
Putting these estimates in the preceding inequality gives:
\begin{align*}
    \v p_t - \Bar{p}_t \v \leq \left[  \frac{\v B \v_2^2}{N}  \kappa e^{2 \kappa R}
    + \frac{ \sigma_{\max}(B^\top)}{N} \kappa e^{2 \kappa R} \v B ( \Bar{z}_1 - y) \v
    + R \frac{\sigma_{\max}(B^\top)}{N} \kappa^2 e^{3 \kappa R}  \v B ( \Bar{z}_1 - y) \v \right] \v v - \Bar{v} \v_{L^2} ,
\end{align*}
which is the desired result.

\end{proof}

\begin{proof}[Proof of~\cref{prop:smoothness}]

Let $v, \Bar{v} \in \Lspace$ with $\v v \v_{L^2}, \v \Bar{v} \v_{L^2} \leq R$. Then taking the same notation as in~\cref{lemma2}, we have for any $t \in \tset$:
\begin{align*}
    \nabla L(v)_t - \nabla L(\Bar{v})_t  & = \sum_{i=1}^N K(., z^i_t)p^i_t - \sum_{i=1}^N K(. \Bar{z}^i_t) \Bar{p}^i_t \\
    &  = \sum_{i=1}^N K(., z^i_t) (p^i_t - \Bar{p}^i_t)
    + \sum_{i=1}^N ( K(., z^i_t) - K(. \Bar{z}^i_t) ) \Bar{p}^i_t ,
\end{align*}
and we can write $\v \nabla L(v)_t - \nabla L(\Bar{v})_t \v_V \leq T_1 + T_2$ with:
\begin{align*}
    T_1 \eqdef \v \sum_{i=1}^N K(., z^i_t) (p^i_t - \Bar{p}^i_t) \v_V, \quad
    T_2 \eqdef \v \sum_{i=1}^N ( K(., z^i_t) - K(. \Bar{z}^i_t) ) \Bar{p}^i_t \v_V .
\end{align*}

First we consider deriving an upper bound on $T_1$. Note that by the definition of the norm in RKHSs and by~\cref{ass:universality} we have:
\begin{align*}
    T_1^2 = \sum_{1 \leq i,j \leq N} (p^i_t - \Bar{p}^i_t)^\top K(z^i_t, z^j_t) (p^j_t - \Bar{p}^j_t) \leq \Lambda \sum_{i=1}^N \v p^i_t - \Bar{p}^i_t \v^2 .
\end{align*}
Therefore, using (ii) from~\cref{lemma2} to bound $\v p^i_t - \Bar{p}^i_t \v$ for every index $i$ we get:
\begin{align*}
    T_1^2 \leq \Lambda \mathbf{C}_1^2 \v v - \Bar{v} \v_{L^2}^2 ,
\end{align*}
with:
\begin{align*}
    \mathbf{C}_1^2 & = \sum_{i=1}^N  \frac{\kappa^2 e^{4 \kappa R} \v B \v_2^2}{N^2} \big[ \v B \v_2 + \v B ( \Bar{z}^i_1 - y ) \v (1 + R e^{\kappa R}) \big]^2 \\
    & \leq \sum_{i=1}^N  \frac{2 \kappa^2 e^{4 \kappa R} \v B \v_2^2}{N^2} \big[ \v B \v_2^2 + \v B ( \Bar{z}^i_1 - y ) \v^2 (1 + R e^{\kappa R})^2 \big] \\
    & \leq \frac{2 \kappa^2 e^{4 \kappa R} \v B \v_2^4}{N} + \frac{4 \kappa^2 e^{4 \kappa R} \v B \v_2^2}{N} (1 + R e^{\kappa R})^2 L(\Bar{v}) ,
\end{align*}
where we recognised $L(\Bar{v})$ in the third line. By continuity of $L$ we can define for every $R \geq 0$:
\begin{align*}
    L^*(R) \eqdef \sup_{\v v \v_{L^2} \leq R} L(v) .
\end{align*}
And therefore:
\begin{align*}
    \mathbf{C}_1^2 & \leq \frac{2 \kappa^2 e^{4 \kappa R} \v B \v_2^4}{N} + \frac{4 \kappa^2 e^{4 \kappa R} \v B \v_2^2}{N} (1 + R e^{\kappa R})^2 L^*(R) =: \mathbf{C}_3(R)^2 .
\end{align*}

We then consider deriving an upper-bound on $T_2$. By triangle inequality:
\begin{align*}
    T_2 \leq \sum_{i=1}^N \v ( K(., z^i_t) - K(.,\Bar{z}^i_t) ) \Bar{p^i}_t \v_V .
\end{align*}
Consider any $\alpha \in V$, then for any index $i \in \llbracket 1, N \rrbracket$, by the reproducing property:
\begin{align*}
    \langle ( K(., z^i_t) - K(., \Bar{z}^i_t)) \Bar{p}^i_t, \alpha \rangle_V & = \langle \alpha(z^i_t) - \alpha(\Bar{z}^i_t), \Bar{p}^i_t \rangle \\
    & \leq \kappa \v \alpha \v_V \v z^i_t - \Bar{z}^i_t \v \v \Bar{p}^i_t \v ,
\end{align*}
where we used the Cauchy-Schwarz inequality and~\cref{ass:admissibility} applied to $\alpha$. Therefore, by duality:
\begin{align*}
    \v ( K(., z^i_t) - K(.,\Bar{z}^i_t) ) \Bar{p^i}_t \v_V \leq \kappa \v z^i_t - \Bar{z}^i_t \v \v \Bar{p}^i_t \v .
\end{align*}
Using the estimates of~\cref{lemma1} and~\cref{lemma2} we get:
\begin{align*}
    \v ( K(., z^i_t) - K(.,\Bar{z}^i_t) ) \Bar{p^i}_t \v_V \leq \frac{\kappa^2 e^{2\kappa R} \v B \v_2}{N} \v B \Bar{z}^i_1 - y^i \v \v v - \Bar{v} \v_{L^2} .
\end{align*}
And finally, using Cauchy-Schwarz inequality and recognizing $L(\Bar{v})$ we have:
\begin{align*}
    T_2^2 & \leq N \sum_{i=1}^N \v ( K(., z^i_t) - K(.,\Bar{z}^i_t) ) \Bar{p^i}_t \v_V^2 \\
    & \leq \mathbf{C}_2^2 \v v - \Bar{v} \v_{L^2}^2 ,
\end{align*}
with:
\begin{align*}
    \mathbf{C}_2^2 & = 2 \kappa^4 e^{4 \kappa R} \v B \v_2^2 L(\Bar{v}) \\
    & \leq 2 \kappa^4 e^{4 \kappa R} \v B \v_2^2 L^*(R) =: \mathbf{C}_4(R)^2 .
\end{align*}

Therefore we obtain the result by setting:
\begin{align*}
    \mathbf{C}(R) = \big[ \Lambda \mathbf{C}_3(R)^2 + \mathbf{C}_4(R)^2 \big]^{1/2} .
\end{align*}

\end{proof}

Provided with~\cref{prop:smoothness}, we can finish the proof of~\cref{thm:main}.

\begin{proof}[Proof of~\cref{thm:main}]

By~\cref{prop:PL_rkhs}, $L$ satisfies the PL inqualities of~\cref{def:PL_gen} and the proof is a direct corollary of~\cref{thm:convergence_gen}. It only remains to show that the smoothness condition of~\cref{def:smoothness} is verified.

Let $v, \Bar{v} \in \Lspace$ such that $\v v \v_{L^2}, \v \Bar{v} \v_{L^2} \leq R$ for some radius $R \geq 0$. Then we have:
\begin{align*}
    L(\Bar{v}) = & L(v) + \int_0^1  \nabla L(v + t(\Bar{v} - v)). (\Bar{v} - v ) \text{d}t \\
    = & L(v) + \nabla L(v) . (\Bar{v} - v ) \\
    & + \int_0^1 \big[ \nabla L(v+t(\Bar{v}-v)) - \nabla L(v) \big] \cdot (\Bar{v} - v) \text{d}t .
\end{align*}
Using~\cref{prop:smoothness}, there exists some $\mathbf{C}(R)$ such that:
\begin{align*}
   \v \nabla L(v+t(\Bar{v}-v)) - \nabla L(v) \v_{L^2} \leq t \mathbf{C}(R) \v \Bar{v} - v \v_{L^2} .
\end{align*}
This gives the inequality:
\begin{align*}
    L(\Bar{v}) \leq L(v) + \nabla L(v) \cdot (\Bar{v} - v ) + \frac{\mathbf{C}(R)}{2} \v \Bar{v} - v \v_{L^2}^2\, ,
\end{align*}
which is the desired result.
\end{proof}

\section{Proofs of~\cref{sec:convergence}} \label{sec:convergence_proof}

The results in~\cref{sec:convergence} show how the condition for convergence in~\cref{rkhs_init_cond} can be enforced by considering suitable RKHSs of vector-fields and suitable matrices $A$ and $B$. We give in~\cref{subsec:kernels} examples of suitable kernels.

In the following, we assume that for every $q \geq 1$ we are provided with a function $k_q : \r_+ \rightarrow \r$ such that the induced symmetric rotationally-invariant kernel $K_q$ defined by:
\begin{align}
    \forall z,z' \in \r^q, \ K_q(z,z') = k_q(\v z-z' \v) \Id_q ,
\end{align}
is a positive-definite kernel over $\r^q$. Without loss of generality, one can assume $k_q$ to be normalized, that is $k_q(0) = 1$. We note $V_q$ the vector-valued RKHS associated with $K_q$. The properties of $V_q$ are then entirely determined by $k_q$. In particular, smoothness of the kernel at $0$ implies regularity of the vector-fields in $V_q$:

\begin{property}[Regularity of $V_q$] \label{prop:kappa}
Let $k_q : \r_+ \rightarrow \r$ be some function defining a positive symmetric kernel $K_q$. If $k_q$ is $4$ times differentiable at $0$, with $k_q'(0) = k_q^{(3)}(0) = 0$. Then $V_q$ satisfies~\cref{ass:admissibility} with constant $\kappa = \sqrt{k_q(0)} + \sqrt{-k_q''(0)} + \sqrt{k_q^{(4)}(0)}$. 
\end{property}

As a consequence, if the derivatives of $k_q$ can be bounded uniformly over $q$ then $V_q$ satisfies~\cref{ass:admissibility} with some constant $\kappa$ independent of $q$. This, is the case for the Matérn kernel $k$ defined in~\cref{matern_kernel}.

\begin{proof}

The proof proceeds by duality arguments. For $q \geq 1$, consider some $v \in V_q$. Then for any $z \in \r^q$ and any $\alpha \in V_q$, by the reproducing properties of RKHSs:
\begin{align*}
    \langle v(z), \alpha \rangle & = \langle v, K_q(.,z) \alpha \rangle_{V_q} \\
    & \leq \v v \v_{V_q} \v K_q(.,z) \alpha \v_{V_q} \\
    & = \v v \v_{V_q} \big( \langle \alpha, K_q(z,z) \alpha \rangle \big)^{1/2} \\
    & \leq \sqrt{k_q(0)} \v v \v_{V_q} \v \alpha \v .
\end{align*}
Therefore, by duality $\v v(z) \v \leq \sqrt{k_q(0)} \v v \v_{V_q}$ and then by taking the supremum over $z \in \r^q$:
\begin{align*}
    \v v \v_{\infty} \leq k_q(0) \v v \v_{V_q}.
\end{align*}
Then for any $z \in \r^q$ any $\alpha, \beta \in \r^q$ and any $h \in \r_+$:
\begin{align*}
    & \langle v(z+h\alpha) - v(z), \beta \rangle  \\
    = \ & \langle v, ( K_q(., z+h\alpha) - K_q(.,z) ) \beta \rangle \\
    \leq \ & \v v \v_{V_q} \v ( K_q(., z+h\alpha) - K_q(.,z) ) \beta \v_{V_q} .
\end{align*}
In the r.h.s we have using Taylor's expansion of $k_q$ at $0$:
\begin{align*}
    \v ( K_q(., z+h\alpha) - K_q(.,z) ) \beta \v_{V_q}^2
    & = 
    \left(
        \begin{array}{c}
            \beta  \\
            - \beta
        \end{array}
    \right)^\top
    \left(
        \begin{array}{cc}
            k_q(0) Id_q & k_q(h \v \alpha \v) Id_q \\
            k_q(h \v \alpha \v) Id_q & k(0) Id_q
        \end{array}
    \right)
    \left(
        \begin{array}{c}
            \beta  \\
            - \beta
        \end{array}
    \right) \\
    & = 2 \v \beta \v^2 ( k_q(0) - k_q(h \v \alpha \v) ) \\
    & = - \v \beta \v^2 h^2 \v \alpha \v^2 k_q''(0) + o(h^2) .
\end{align*}
Taking the limit $h \rightarrow 0$:
\begin{align*}
    \langle Dv(z) \alpha, \beta \rangle & = \lim_{h \rightarrow 0} h^{-1} \langle v(z+h\alpha) - v(z), \beta \rangle \\
    & \leq \sqrt{-k_q''(0)} \v v \v_{V_q} \v \alpha \v \v \beta \v ,
\end{align*}
and therefore $\v Dv(z) \v_2 \leq \sqrt{-k_q''(0)} \v v \v_{V_q} $.

Finally, let us bound $\v D^2 v \v_{2, \infty}$. For any $z \in \r^q$ any $\alpha, \beta, \gamma \in \r^q$ and any $h,l \geq 0$ we have in the same manner:
\begin{align*}
   & \langle v(z+h\beta + l \alpha) - v(z+h\beta) - v(z+l\alpha) + v(z), \gamma \rangle \\
     & \quad \leq \v v \v_{V_q} \v \beta \v \v \alpha \v \v \gamma \v h l \sqrt{k_q^{(4)}(0)} + o (h l) 
\end{align*}
where the second line is obtained by Taylor expansion of $k_q$ at $0$. Thus, taking the limit $h,l \rightarrow 0$:
\begin{align*}
    \langle D^2v(z)(\alpha, \beta), \gamma \rangle & =  \lim_{h,l \rightarrow 0} h^{-1} l^{-1} \langle v(z+h\beta + l \alpha) - v(z+h\beta) - v(z+l\alpha) + v(z), \gamma \rangle \\
    & \leq \sqrt{k_q^{(4)}(0)} \v v \v_{V_q} \v \beta \v \v \alpha \v \v \gamma \v ,
\end{align*}
and therefore $\v D^2 v(z) \v_2 \leq \sqrt{k_q^{(4)}(0)} \v v \v_{V_q} $.

Setting $\kappa = \sqrt{k_q(0)} + \sqrt{-k_q''(0)} + \sqrt{k_q^{(4)}(0)}$ we obtain the result. Moreover, choosing appropriate $v$ in the above proof, inequalities become sharp and one observes that the constant $\kappa$ is optimal.

\end{proof}

\subsection{Enforcing convergence with high dimensional lifting and universal kernels} \label{subsec:proof_q}

Here we investigate the dependency of~\cref{rkhs_init_cond} w.r.t. $q$, $\delta$ and $N$ for the class of RKHS $V_q$ and thereby recover the proof of~\cref{prop:embedding}.

We make the following assumption concerning the decay of $k_q$ at infinity:
\begin{assumption}[Decay of $k_q$]\label{ass:decrease}
    For every $q \geq 1$, $k_q(x)$ tends to $0$ when $x$ tends to infinity and we note $\beta_{q,N} > 0$ s.t.:
    \begin{align*}
        \forall x \geq \beta_{q,N}, \ | k_q(x)| \leq \frac{1}{2N} .
    \end{align*}
    Moreover for fixed $N$ we assume that
    \begin{align*}
        \beta_{q,N} = o_{q \rightarrow + \infty} (q^{1/4}).
    \end{align*}
\end{assumption}

\subsubsection{Lifting matrices} \label{subsubsec:lifting}

For any $q \geq 1$ we consider here the matrices:
\begin{align*}
    A_q & \eqdef q^{-1/4}(\Id_d, ..., \Id_d, 0 )^\top \in \r^{q \times d}, \\
    B_q & \eqdef q^{1/4} (\Id_{d'}, 0 ... 0 ) \in \r^{d' \times q} ,
\end{align*}
where there are $\lfloor q/d \rfloor$ copies of $\Id_d$ in $A_q$. In particular we have:
\begin{align*}
    \sigma_{\min}(A_q) & = q^{-1/4} \sqrt{\lfloor q/d \rfloor} \simeq q^{1/4}, \\
    \sigma_{\min}(B_q^\top) & = \sigma_{\max}(B_q^\top) = q^{1/4}
\end{align*}
and $B_q A_q \in \r^{d' \times d}$ is independent of $q$. We also consider for every $q \geq 1$ some control parameter initialization $V^0_q \in L^2(V_q)$ such that $\v v^0_q \v_{L^2} \leq R_0 q^{-1/4}$ and assume the data distribution to be compactly supported.

\begin{proposition} \label{prop:embedding_gen}
Let $R > 0$ and $d, d' \geq 1$.
Assume~\cref{ass:decrease} is satisfied, $V_q$ satisfies~\cref{ass:admissibility} with constant $\kappa$ independent of $q$ and there exists $R_0 > 0$ s.t. $\v v^0_q \v \leq R_0 q^{-1/4}$ for every $q \geq 1$. \\
Then there exists some constant $C > 0$ so that for any $N \geq 2$ and any $\delta \in (0,1]$, \cref{rkhs_init_cond} is satisfied with matrices $A_q, B_q$ and $\kappa, \lambda, \Lambda$ associated with the RKHS $V_q$ as soon as:
\begin{align} \label{bound_gen_q}
    q \geq C N^4, \ \text{and} \ q \geq C \delta^{-4}\beta_{q,N}^4 .
\end{align}
\end{proposition}

Note that the second condition in~\cref{bound_gen_q} can always be ensured for large enough $q$ thanks to~\cref{ass:decrease}. In the case of the Matérn kernel $k$ defined in~\cref{matern_kernel}, such an assumption is verified because it has exponential decay and it is independent of $q$. Hence, \cref{prop:embedding} is a direct consequence of~\cref{prop:embedding_gen}.

\begin{proof}[Proof of~\cref{prop:embedding_gen}]
Let $q \geq 1$. Using the fact that $d^2 \lfloor q/d \rfloor^2 \geq q(q-2d)$, considering:
\begin{align} \label{embedding_gen_qbound}
    q \geq 2d + d^2 \frac{\beta_{q,N}^4}{\delta^4 e^{- 4 \kappa (R+R_0)}} 
\end{align}
is enough to ensure that:
\begin{align*}
    q^{-1/4} \sqrt{\lfloor q/d \rfloor} \delta e^{- \kappa (R+R_0)} \geq \beta_{q,N} .
\end{align*}
Then, by~\cref{ass:decrease} for $(z^i)_{1 \leq i \leq N} \in (\r^q )^N$ with data separation $q^{-1/4} \sqrt{ \lfloor q/d \rfloor } \delta e^{- \kappa ( R+R_0)}$ we have:
\begin{align*}
    \forall 1 \leq i < j \leq N, \ | k_q ( \v z^i - z^j \v ) | \leq \frac{1}{2N} .
\end{align*}
Thus, the kernel matrix $\mathbb{K} = (k_q (\v z^i - z^j \v) \Id_q)_{i,j}$ is diagonally dominant with:
\begin{align*}
    \lambda_{\min}(\mathbb{K}) \geq 1 - \frac{N-1}{2N} \geq \frac{1}{2} ,
\end{align*}
and by definition of $\lambda$ in~\cref{rkhs_init_cond}:
\begin{align} \label{lambda_bound2}
    \lambda ( \sigma_{\min}(A_q)^{-1} \delta^{-1} e^{ \kappa (R+R_0)}  ) \geq \frac{1}{2}.
\end{align}
Moreover, $\Lambda \leq N$ because $k_q$ is bounded by $1$.

Let $x \in B(0,r_0)$ and assume $z$ is a solution of~\cref{rkhs_forward} for the control parameter $v^0_q$ and with initial condition $A_q x$. We have at time $t=1$:
\begin{align*}
    z_1 = A_q x + \int_0^1 (v^0_q)_t(z_t) \text{d}t ,
\end{align*}
so that by triangle inequality and~\cref{ass:admissibility}:
\begin{align*}
    \v z_1 - A_q x \v \leq \kappa \v v^0_q \v_{L^2} ,
\end{align*}
and then because $\v v^0_q \v \leq R_0 q^{-1/4}$ and the dataset is compactly supported:
\begin{align*}
    \v F(v^0_q,x) \v & = \v B_q z_1 \v \\
    & \leq \v B_q A_q x \v + \v B_q (z_1 - A_q x) \v \\
    & \leq \v B_q A_q \v_2 r_0 + \kappa R_0 ,
\end{align*}
with $B_q A_q$ independent of $q$. Thus $L(v^0_q) \leq C$ for some constant $C$ independent of $q$, $N$ and $\delta$.

Finally:
\begin{align} \label{sigma_bound2}
    \frac{\sigma_{\max}(B_q^\top)}{\sigma_{\min}(B_q^\top)^2} = q^{-1/4} ,
\end{align}
and putting~\cref{lambda_bound2} and~\cref{sigma_bound2} into the l.h.s.~\cref{rkhs_init_cond} gives:
\begin{align*}
    \frac{2\sqrt{2} \sigma_{\max}(B_q^\top) \sqrt{N \Lambda L(0)}  e^{3 \kappa (R+R_0)}}{\sigma_{\min}(B_q^\top)^2 \lambda (\sigma_{\min}(A_q)^{-1} \delta^{-1} e^{- \kappa (R+R_0)}) }
    \leq 4 \sqrt{2 C} e^{3 \kappa (R+R_0)} \frac{N}{q^{1/4}} .
\end{align*}
Considering $R > 0$ is fixed (c.f.~\cref{rq:choice_R}), \cref{thm:main} can be applied as soon as:
\begin{align} \label{embedding_gen_qbound2}
    q \geq  2^{10} C^2 e^{12 \kappa (R+R_0)} R^{-4} N^4 
\end{align}
and combining this bound with the one in~\cref{embedding_gen_qbound} gives the result.
\end{proof}

\begin{remark}[Choice of $R$] \label{rq:choice_R}
The proof of~\cref{prop:embedding_gen} holds for any fixed $R > 0$ whose choice impacts the result through the constant $C$. There is a trade-off between minimizing $e^{4 \kappa R}$ to have a better dependency of $q$ w.r.t. $\delta^{-1} \log(N)$ in~\cref{embedding_gen_qbound} and minimizing $R^{-1} e^{3 \kappa R}$ to have a better dependency w.r.t. $N$ in~\cref{embedding_gen_qbound2}. However, in any case, optimizing w.r.t. $R$ only improves the result up to a constant factor.
\end{remark}

\subsubsection{Scaling matrices} \label{subsubsec:scaling}
For $\alpha > 0$, we consider here the matrices:
\begin{align*}
    A = \alpha (\Id_d, 0)^\top \in \r^{(d+d') \times d} \quad \text{and} \quad B = \alpha (0, \Id_{d'}) \in \r^{d' \times (d+d')} .
\end{align*}
Then, in the proof of~\cref{prop:embedding_gen} one has $\sigma_{\min}(A) = \alpha$ and thus~\cref{lambda_bound2} holds as soon as:
\begin{align*}
    \alpha \geq \delta^{-1} e^{\kappa (R+R_0)} \beta_{d+d', N} .
\end{align*}
Moreover, $\sigma_{\max}(B^\top) = \sigma_{\min}(B^\top) = \alpha$ and $F(0, x) = 0$ for every input $x$ as $B A = 0$. Thus, with initialization $v^0 = 0$ the l.h.s. of~\cref{rkhs_init_cond} scales as:
\begin{align*}
    \frac{2\sqrt{2} \sigma_{\max}(B^\top) \sqrt{N \Lambda L(0)}  e^{3 \kappa R)}}{\sigma_{\min}(B^\top)^2 \lambda (\sigma_{\min}(A)^{-1} \delta^{-1} e^{- \kappa R}) }
    \leq 4 \sqrt{2 C} e^{3 \kappa R} \frac{N}{\alpha} = \mathcal{O}(1/\alpha) ,
\end{align*}
and global convergence holds for $\alpha = \Omega (\delta^{-1} \beta_{d+d', N} + N)$.

%Although it seems to be a very simple alternative the reason for us not to propose it is that model learned with those parameters is not likely to generalize. Indeed, when $B A = 0$ then $F(0, x) = 0$ for every input $x$ and \cref{rkhs_init_cond} is not invariant by rescaling of $A$ and $B$ anymore. Although interpolation of the training dataset can then be achieved for a suitable choice of scaling of the parameter, it has also been shown that this ``lazy regime'' can lead to bad generalization properties~\cite{chizat2019lazy}.

\subsection{Enforcing convergence with high dimensional embedding en finite dimensional kernels} \label{subsec:proof_qint}

We recover here the result of~\cref{prop:qint} for the more general kernel $k_q$. In particular notice that, as an application of Bochner's theorem~\cite{rudin2017fourier}, for every $q \geq 1$ there exists some probability measure $\mu_q$ over $\r^q$ such that:
\begin{align}
    \forall z \in \r^q, \ k_q( \v z \v ) = \int_{\r^q} e^{\imath \langle z, \omega \rangle} \text{d}\mu_q(\omega) .
\end{align}
Then, such as in~\cref{RFF} for the Matérn kernel, for any independent sampling $\omega^j \sim \mu_q$ of size $\qint$ one can consider the feature map:
\begin{align}
    \varphi : z \mapsto \left( e^{\imath \langle z, \omega^j \rangle } \right)_{1 \leq j \leq \qint} \in \mathbb{C}^{\qint} .
\end{align}
Such a feature map induces a structure of RKHS $\Hat{V}_q$ which is the set of residuals of~\cref{residual_set} with activation $\varphi$. The associated kernel is $\Hat{K}_q : (z,z') \mapsto \Hat{k}_q(z,z') \Id_q$ with:
\begin{align*}
    \forall z,z' \in \r^q, \ \Hat{k}_q(z,z') & \eqdef \langle \varphi(z), \varphi(z') \rangle \\
    & \xrightarrow{\qint \rightarrow +\infty} k_q(\v z-z' \v),
\end{align*}
almost surely, by the law of large numbers.

We make the following assumption on $\mu_q$:
\begin{assumption}[Moments of $\mu_q$] \label{ass:moments}
    The measure $\mu_q$ admits finite moments up to order $8$:
    \begin{align*}
        \mathbb{E}_{\mu_q} \left[ \prod_{j=1}^8 \left| \omega_{i_j} \right| \right] < \infty, \ \forall i_1, ..., i_8 \in \llbracket 1, q \rrbracket .
    \end{align*}
    Moreover, we assume those moments are independent of $q$.
\end{assumption}

Note that~\cref{ass:moments} implies regularity on the function $k_q$. Indeed by Fourier inversion theorem we have for every $r \in \r_+$ and every $\theta \in \mathbb{S}^{d-1}$:
\begin{align*}
    k_q(r) = \mathbb{E}_{\mu_q} \left[ e^{\imath r \langle \theta, \omega \rangle} \right] .
\end{align*}
By theorems of derivation under the integral $k_q$ is $8^{th}$-time differentiable on $\r_+$ and for $0 \leq  l \leq 8$:
\begin{align*}
    k_q^{(l)}(r) = \mathbb{E}_{\mu_q} \left[ (\imath \langle \theta, \omega \rangle)^l e^{\imath r \langle \theta, \omega \rangle} \right] .
\end{align*}
In particular, $k_q$ is four time differentiable at $0$ and:
\begin{align*}
    k'(0) & = \mathbb{E}_{\mu_q} \left[ \imath \langle \theta, \omega \rangle  \right] \\
    k^{(3)}(0) & = \mathbb{E}_{\mu_q} \left[ - \imath \langle \theta, \omega \rangle^3  \right]
\end{align*}
Therefore, $k_q'(0)$ and $k_q^{(3)}(0)$ are in $\imath \r \cap \r = \lbrace 0 \rbrace$ and~\cref{prop:kappa} holds. Moreover, as the moments are independent of $q$, the associated $\kappa$ is also independent of $q$.

\begin{proposition} \label{prop:qint_gen}
Consider $q, N \geq 1$ and $\epsilon, \tau, R > 0$.

\textbf{(i)} Assume~\cref{ass:moments} is satisfied. For $\qint \geq \Omega(\tau q^8)$, with probability greater than $1-\tau^{-1}$, $\Hat{V}_q$ satisfies~\cref{ass:admissibility} with some $\Hat{\kappa} \leq \kappa + 1$.

\textbf{(ii)} For $\qint \geq \Omega(\epsilon^{-2} N^2 (q \log(\v A \v_2 r_0 + R) + \tau)) $, with probability greater than $1 - e^{-\tau}$, for any control parameter $v \in L^2(\left[ 0,1 \right], \Hat{V}_q)$ s.t. $\v v \v_{L^2} \leq R$ and any time $t \in \tset$:
\begin{align*}
    \lambda_{\min}(\Hat{\mathbb{K}} ( (z^i_t)_i ) ) \geq \lambda_{\min}(\mathbb{K} ( (z^i_t)_i ) ) - \epsilon ,
\end{align*}
where the $(z^i)_i$ are the solutions to~\cref{rkhs_forward} and $\Hat{\mathbb{K}}$, $\mathbb{K}$ are the kernel matrices associated with $\Hat{k}$ and $k$ respectively.
 %   \end{enumerate}
\end{proposition}

As~\cref{ass:moments} is satisfied for the Matérn kernel $k$ defined in~\cref{matern_kernel} as soon as $\nu > 4$, \cref{prop:qint} is a direct consequence of~\cref{prop:qint_gen}.

\begin{proof}[Proof of~\cref{prop:qint_gen}]

\textbf{Proof of (i)}
We already saw that thanks to the assumption on the moments of $\mu_q$, the RKHS $V_q$ associated with $k_q$ satisfies~\cref{ass:admissibility} with constant $\kappa$.

Then we want to prove that for sufficiently high $\qint$, the RKHS $\Hat{V}_q$ generated by the feature map $\varphi$ in~\cref{RFF}, satisfies~\cref{ass:admissibility}.

Let $v \in \Hat{V}_q$ be of the form:
\begin{align*}
    v : z \mapsto W \varphi(z)
\end{align*}
for some $W \in \r^{q \times \qint}$. For $z \in \r^q$, $\v \varphi (z) \v = 1$ and thus:
\begin{align*}
    \v v(z) \v & = \v W \varphi(z) \v  \leq \v W \v = \v v \v_{\Hat{V}_q} ,
\end{align*}
so that $\v v \v_{\infty} \leq \v v \v_{\Hat{V}_q}$.

Then $Dv(z) = W D\varphi(z)$ and by the law of large number we have for any $\theta \in \mathbb{S}^{q-1}$:
\begin{align*}
    \v D\varphi(z) \theta \v^2 & = \frac{1}{\qint} \sum_{j=1}^{\qint}  \sum_{1 \leq k,l \leq q} \omega^j_k \omega^j_l \theta_k \theta_l \\
    & = \frac{1}{\qint} \sum_{j=1}^{\qint}  \langle \omega^j, \theta \rangle^2 \\
    & \xrightarrow[]{\qint \rightarrow +\infty} \mathbb{E}_{\mu_q} \left[ \langle \omega, \theta \rangle^2 \right] = - k_q''(0) .
\end{align*}
Because $\mu_q$ admits finite fourth order moments, the rate of convergence can be controlled using Chebyshev's inequality. For every indices $k,l \in \llbracket 1, q \rrbracket$:
\begin{align*}
    \mathbb{P} \big( \left| \frac{1}{\qint} \sum_{j=1}^{\qint} \omega^j_k \omega^j_l - \mathbb{E}_{\mu_q} \left[ \omega_k \omega_l \right] \right| \geq \alpha/q \big) \leq \frac{q^2 \mathbb{E}_{\mu_q} \left[ \omega_k^2 \omega_l^2 \right]}{\alpha^2 \qint} .
\end{align*}
For $\qint \geq \Omega( \frac{q^4 \tau }{\alpha^2} ) $ we have with probability greater than $1 - \tau^{-1}$ that the above inequality is satisfied for every indices $k,l$. Thus for every $z \in \r^q$ and every $\theta \in \mathbb{S}^{q-1}$:
\begin{align*}
    \left| \v D\varphi(z)\theta \v^2 + k_q''(0) \right|
    & \leq \sum_{1 \leq k,l \leq q} | \theta_k\theta_l | 
    \Big| 
    \frac{1}{r} \sum_{j=1}^r  \omega^j_k \omega^j_l - \mathbb{E}_{\mu_q} \left[ \omega_k \omega_l \right] 
    \Big| \\
    &  \leq \sum_{1 \leq k,l \leq q} | \theta_k \theta_l | \frac{\alpha}{q} \\
    & \leq \alpha,
\end{align*}
using Chauchy-Schwarz inequality in the last line. We can thus conclude:
\begin{align*}
    \v D \varphi \v_{2,\infty}^2 \leq -k_q''(0) + \alpha .
\end{align*}

The same arguments holds for $D^2v(z) = W D^2 \varphi(z)$. For any $\theta \in \mathbb{S}^{q-1}$ we have:
\begin{align*}
    D^2 \varphi(z)(\theta,\theta) = \left( \frac{1}{\sqrt{\qint}} \sum_{1 \leq k,l \leq q} - e^{\imath \langle z, \omega^j \rangle } \omega^j_k \omega^j_l \theta_k \theta_l \right)_{1 \leq j \leq \qint} .
\end{align*}
Passing to the squared norm we get:
\begin{align*}
    \v  D^2 \varphi(z)(\theta,\theta) \v^2
    & = \frac{1}{\qint}\sum_{j=1}^{\qint} \sum_{1 \leq k,l,s,t \leq q} \omega^j_k \omega^j_l \omega^j_s \omega^j_t \theta_k \theta_l \theta_s \theta_t \\
    & \xrightarrow[]{\qint \rightarrow +\infty} \sum_{1 \leq k,l,s,t \leq q} \mathbb{E}_{\mu_q} \left[ \omega_k \omega_l \omega_s \omega_t \right] \theta_k \theta_l \theta_s \theta_t \\
    & = \mathbb{E}_{\mu_q} \left[ \langle \omega, \theta \rangle^4 \right] =  k_q^{(4)}(0) .
\end{align*}
Then because $\mu_q$ admits $8^{th}$ order moments, we can control the convergence in probability by Chebyshev's inequality. For $\qint \geq \Omega ( \frac{q^8 \tau }{\alpha^2 } )$ we have with probability greater than $1- \tau^{-1}$:
\begin{align*}
    \v D^2 \varphi \v_{2,\infty}^2 \leq k_q^{(4)}(0) + \alpha .
\end{align*}

Finally $\Hat{V}_q$ satisfies~\cref{ass:admissibility} with:
\begin{align*}
    \Hat{\kappa} \leq (k_q(0))^{1/2} + (-k_q''(0))^{1/2} + (k_q^{(4)}(0))^{1/2} + 1
\end{align*}
for $\alpha$ sufficiently low.

\textbf{Proof of (ii). }
For $t \in \tset$, we consider $(z_t^i)_i$ the solutions of of~\cref{rkhs_forward} for some control parameter $v \in L^2(\tset, \Hat{V}_q)$ and we introduce the kernel matrices:
\begin{align*}
    \Hat{\mathbb{K}}_t  = ( \Hat{K}_q (z_t^i, z_t^j) )_{1 \leq i,j \leq N},  \:
    \mathbb{K}_t = ( K_q(z_t^i, z_t^j) )_{1 \leq i,j \leq N}.
\end{align*}
Using the first point, we know that if $\v v \v_{L^2} \leq R$, then
$\v z^i_t \v \leq \v A x^i \v + (\kappa+1)R$.
Then, using Theorem 1 in~\cite{sriperumbudur2015optimal}, we have for every indices $i,j$ and every $t \in \tset$:
\begin{align*}
    \mathbb{P} \Big( | \Hat{k}(z^i_t, z^j_t) - k(\v z^i_t - z^j_t \v) | \geq \frac{h(q,R) + \sqrt{2\tau}}{\sqrt{\qint}} \Big) \leq e^{-\tau} ,
\end{align*}
with $h(q,R) \eqdef \mathcal{O}(\sqrt{q \log(\v A \v_2 r_0 + R)})$. Thus, choosing $\qint \geq \Omega \left( \epsilon^{-2} N^2 ( q \log(\v A \v_2 r_0 + R) + \tau ) \right)$, we have with probability greater than $1 - e^{- \tau}$, $\lambda_{\min}(\Hat{\mathbb{K}}_t) \geq \lambda_{\min} \left( \mathbb{K}_t \right) - \epsilon$, for any $t \in \tset$ .

\end{proof}

Note that the assumption of finite $8^{th}$ moments is only needed to control the convergence rate of $\Hat{k}_q$ towards $k_q$ in probability. By the law of large numbers, assuming finite $4^{th}$-order moments is sufficient to have convergence almost surely. Also, we used the Chebyshev's inequality in order to control the convergence rate. Making stronger assumptions on the decay of $\mu_q$ (e.g. sub-gaussianity) could have led to faster convergence by using sharper concentration inequalities.

\subsection{Example of appropriate kernels} \label{subsec:kernels}
We show here that the Matérn kernel of parameter $\nu \in (8, +\infty]$ satisfies~\cref{ass:decrease} and~\cref{ass:moments}.

\myparagraph{Gaussian kernel.}
The Gaussian kernel defined by for some parameter $\sigma > 0$ by $k_q(r) = e^{- \frac{\sigma^2 r^2}{2}}$. In this case the frequency distribution $\mu_q$ is the multivariate normal of variance $\sigma$ and has a density given for every $\omega \in \r^q$ by:
\begin{align*}
    \mu_q(\omega) = \frac{1}{(2 \pi \sigma^2)^{q/2}} e^{- \frac{\v \omega \v^2}{2 \sigma^2} } ,
\end{align*}
This distribution admits finite moments of every order which are independent of $q$. Also, $k_q$ is four times differentiable at $0$ and by~\cref{prop:kappa} the associated $V_q$ is (strongly) admissible with
$
    \kappa = 2 + \sqrt{3}
$

Moreover~\cref{ass:decrease} as one has $| k_q(x) | \leq 1/2N$ if:
\begin{align*}
    x \geq \beta_{q,N} = \frac{2}{\sigma^2} \sqrt{\log(2N)}.
\end{align*}

\myparagraph{Matérn kernel.}
Sobolev spaces $H^s(\r^q,\r^q)$ are RKHSs as soon as $s > q/2$. Given some $\nu > 0$, the kernel $k_q$ associated with $H^{(q/2 + \nu)}(\r^q,\r^q)$ is independent of $q$ and is defined in~\cref{matern_kernel}. It is associated with the multivariate t-distribution:
\begin{align*}
    \mu_q( \omega ) = C(q,\nu) ( 1 + \frac{\v \omega \v^2}{2 \nu} )^{- (\nu + q/2)} ,
\end{align*}
for some normalising constant $C(q,\nu)$. Therefore, $\mu_q$ admits $l^{th}$ order moments as soon as $\nu \geq l/2 $, and those moments are bounded independently of $q$ (see~\cite{review_multivariate} for the computation of moments). In particular, for $\nu > 2$, $k_q$ is four times differentiable at $0$ with $k''(0) = \nu/(\nu-1)$ and $k^{(4)}(0) = 3\nu^2/(\nu-1)(\nu-2)$. Thus by~\cref{prop:kappa}, $V_q$ is (strongly) admissible with:
\begin{align*}
    \kappa = 1 + \sqrt{\frac{\nu}{(\nu-1)}} + \sqrt{\frac{3 \nu^2}{(\nu-1)(\nu-2)}} .
\end{align*}

Because $k_q$ has exponential decay (see~\cite{laforgia2010some}), there exist constants $H_\nu, G_\nu$ such that:
\begin{align*}
     | k_q(r) | \leq G_\nu e^{- H_\nu^{-1} r} 
\end{align*}
and~\cref{ass:decrease} is satisfied with
\begin{align*}
    \beta_{q,N} = H_\nu \log(2 G_\nu N).
\end{align*}

\begin{remark}[Sampling] \label{rk:sampling}
Sampling over $\mu_q$ can be achieved using that for $Y \sim \mathcal{N}(0,\Id_q)$ and for $u$ distributed according to $\chi^2_{2 \nu}$, the chi-squared distribution with $2 \nu$ degrees of freedom, $Y/\sqrt{u/2\nu}$ is distributed according to $\mu_q$.  
\end{remark}

\section{\modelname as a generalization of linear networks} \label{sec:linear}

In an attempt to better understand the convergence properties of GD in the training of ResNets, lots of attention has first been brought towards the study of linear models, for which the training dynamic is now well understood~\cite{hardt_identity_2018, bartlett2018gradient, zou2019global}. We explain here in what extent our work can be seen, at least formally, as a generalization of these results to a more general class of ResNets. In this purpose, we highlight the similarity between~\cref{thm:main}, which applies to the whole class of models described by~\cref{def:rkhs_resnet}, and~\cite[Theorem 3.1.]{zou2019global}, which only applies to linear ResNets.

More precisely, \cite{zou2019global} studies model of the form:
\begin{align} \label{linear}
    F(W,x) \eqdef B (\Id + \frac{1}{D} W_D) ... (\Id + \frac{1}{D} W_1) A x ,
\end{align}
where $x \in \r^{d}$ is the input data, $W = (W_1, ..., W_D) \in (\r^{q \times q})^D$ is the trained parameter and $A \in \r^{q \times d}$, $B \in \r^{d' \times q}$ are fixed matrices. Taking the limit of infinite depth $D \rightarrow +\infty$ in the above model motivates the following definition for linear Neural ODE models:

\begin{definition}[Linear-NODE] \label{def:linear}
Let $A \in \r^{q\times d}$ and $B \in \r^{d' \times q}$ be fixed matrices. Then for $W \in L^2([0,1], \r^{q \times q})$ and input $x \in \r^d$, the Linear-NODE output is given by $F(W, x) \eqdef B U_1 A x$, where $U$ is the solution to the following \emph{forward problem}:
\begin{align*}
    \Dot{U}_t = W_t U_t, \quad \text{and} \quad U_0 = \Id_{\r^q} .
\end{align*}
\end{definition}

One sees that the ResNet $F$ has residual terms that are linear w.r.t. the parameters and thus fits in the framework of our analysis. More precisely, the Linear-NODE of~\cref{def:linear} can be seen as a special instance of \modelname of~\cref{def:rkhs_resnet} with space of residual defined as:
\begin{align*}
    V \eqdef \lbrace v : z \mapsto W z, \ W \in \r^{q \times q} \rbrace .
\end{align*}
This corresponds to~\cref{residual_set} with the choice of feature map $\varphi = \Id :\r^q \rightarrow \r^q$. The set of residuals $V$ is then of course a RKHS for the Frobenius metric on matrices. In particular $V$ satisfies  an analog of~\cref{ass:admissibility} in the sense that for $(v : z \mapsto Wz) \in V$:
\begin{align*}
    \max \lbrace \sup_{\v z \v = 1} \Vert v(z) \Vert, \sup_{\v z \v =1} \Vert Dv (z) \Vert, \sup_{\v z \v = 1} \Vert D^2 v(z) \Vert \rbrace \leq \Vert W \Vert = \Vert v \Vert_V . 
\end{align*}
Universality (\cref{ass:universality}) is also satisfied on full-rank data matrices. If $Z = (z^1|...|z^N) \in \r^{q \times N}$ then the associated kernel matrix verifies:
\begin{align*}
    \lambda_{\min} ( \mathbb{K}((z^i)) ) & = \lambda_{\min} ( Z^\top Z ) = \sigma_{\min}(Z)^2, \\
    \lambda_{\max} ( \mathbb{K}((z^i)) ) & = \lambda_{\max} ( Z^\top Z ) = \sigma_{\max}(Z)^2 .
\end{align*}

As in our above presentation we consider training Linear-NODE for the minimization of the empirical risk associated to the square euclidean distance on the output space $\r^{d'}$. Given data matrices $X = (x^1 | ... | x^N) \in  \r^{d \times N}$ for the input and $Y = (y^1|...|y^N) \in \r^{d' \times N}$ for the output, we aim at finding a control parameter minimizing the risk defined for every $W \in L^2([0,1], \r^{q \times q})$ as:
\begin{align*}
    L(W) \eqdef \frac{1}{2N} \sum_{i=1}^N \Vert F(W,x^i) - y^i \Vert =  \frac{1}{2N} \Vert B U_1 A X - Y \Vert^2 .
\end{align*}
One difference with the previous analysis is that one can not expect the empirical risk to reach the value~$0$ if the target data $Y$ is not in the linear span of the input $X$. We are thus interested in minimizing the excess risk defined as:
\begin{align*}
    \Tilde{L}(W) \eqdef L(W) - L^*
\end{align*}
with $L^* \eqdef \inf_{U \in \r^{q \times q}} \frac{1}{2N} \Vert BUAX - Y \Vert^2$.

Following the line of the proof of~\cref{prop:PL_rkhs}, one can then show that the excess risk $\Tilde{L}$ associated to our Linear-NODE model verifies the following (local) PL property:
\begin{align*}
    \forall W \in L^2([0,1], \r^{q \times q}), \quad 2 m(\Vert W \Vert) \Tilde{L}(W) \leq \Vert \nabla \Tilde{L}(W) \Vert^2 \leq 2 M(\Vert \Tilde{W} \Vert) \Tilde{L}(W),
\end{align*}
where $m$ and $M$ are given for $R \geq 0$ by:
\begin{align*}
    m(R) = \frac{1}{N} \sigma_{\min}(B^\top)^2 \sigma_{\min}(A)^2 \sigma_{r}(X)^2 e^{-2R}, \quad
    M(R) =  \frac{1}{N} \sigma_{\max}(B^\top)^2 \sigma_{\max}(A)^2 \sigma_{\max}(X)^2 e^{2R} ,
\end{align*}
with $\sigma_r(X)$ the smallest positive singular value of $X$. Hence, in the same way local PL implies local convergence for a general RKHS $V$ (\cref{thm:main}), convergence in the linear case follows as an application of~\cref{thm:convergence_gen}:

\begin{theorem}[analog to Theorem 3.1. in~\cite{zou2019global}]
Let $W_0$ be some control parameter initialization with norm $\Vert W_0 \Vert = R_0$ and assume there exists some $R > 0$ s.t.:
\begin{align*}
    \sqrt{8} \frac{\sigma_{\max}(B^\top) \sigma_{\max}(A) \sigma_{\max}(X) }{\sigma_{\min}(B^\top)^2 \sigma_{\min}(A)^2 \sigma_r(X)^2 } \sqrt{ L(W_0) - L^*} \leq R e^{-3(R+R_0)}
\end{align*}
then, for a sufficiently small step-size $\eta$, GD initialized at $W_0$ converges towards a global minimizer of $L$ with linear convergence rate.
\end{theorem}

\end{document}